\documentclass[final,onefignum,onetabnum]{siamonline181217}

\usepackage{lipsum}
\usepackage{amsfonts}
\usepackage{graphicx}
\usepackage{epstopdf}
\usepackage{amsmath,amsfonts,amssymb,bm,mathrsfs}
\usepackage{caption}
\usepackage{graphicx, subfigure}
\usepackage{color}
\usepackage{booktabs}
\usepackage{algorithm}
\usepackage{setspace}
\usepackage{algpseudocode}
\usepackage{epsfig}
\usepackage{float}
\usepackage{cite}
\usepackage{exscale}
\usepackage{relsize}
\usepackage{indentfirst}
\usepackage{multirow}
\usepackage{hyperref}
\graphicspath{{report_image/}}
\ifpdf
  \DeclareGraphicsExtensions{.eps,.pdf,.png,.jpg}
\else
  \DeclareGraphicsExtensions{.eps}
\fi

\usepackage{enumitem}
\setlist[enumerate]{leftmargin=.5in}
\setlist[itemize]{leftmargin=.5in}

\newsiamremark{remark}{Remark}
\newsiamremark{hypothesis}{Hypothesis}
\crefname{hypothesis}{Hypothesis}{Hypotheses}
\newsiamthm{claim}{Claim}

\headers{CURE: Curvature Regularization For Missing Data Recovery}{Bin Dong, Haocheng Ju, Yiping Lu, and Zuoqiang Shi}

\author{XXX\thanks{XXX 
  (\email{XXX}, \url{XXX}).}
\and XXX\thanks{XXX
  (\email{XXX}, \email{XXX}).}
\and XXX\footnotemark[3]}

\usepackage{amsopn}
\DeclareMathOperator{\diag}{diag}

\usepackage{parskip}                
\usepackage[thmmarks,amsmath]{ntheorem}
\usepackage{multirow}
\usepackage{amssymb,amsfonts}

\title{CURE: Curvature Regularization For Missing Data Recovery}

\author{Bin Dong \thanks{Beijing International Center for Mathematical Research, Peking University, Beijing, 100871 China.(\email{dongbin@math.pku.edu.cn})}
\and Haocheng Ju\thanks{School Of Mathematical Science, Peking University, Beijing, 100871 China.(\email{jhc9912@pku.edu.cn})}
\and Yiping Lu\thanks{Institute for Computational and Mathematical Engineering (ICME), Stanford University, Stanford, CA, 94305.(\email{yplu@stanford.edu})}
\and Zuoqiang Shi\thanks{Department of Mathematical Sciences, Yau Mathematical Sciences Center, Tsinghua University, Beijing, 100084 China.}(\email{zqshi@tsinghua.edu.cn})}

\begin{document}

\maketitle
\begin{abstract}
Missing data recovery is an important and yet challenging problem in imaging and data science. Successful models often adopt certain carefully chosen regularization. Recently, the low dimensional manifold model (LDMM) was introduced by \cite{osher2016a} and shown effective in image inpainting. The authors of \cite{osher2016a} observed that enforcing low dimensionality on image patch manifold serves as a good image regularizer. In this paper, we observe that having only the low dimensional manifold regularization is not enough sometimes, and we need smoothness as well. For that, we introduce a new regularization by combining the low dimensional manifold regularization with a higher order \textbf{CU}rvature \textbf{RE}gularization, and we call this new regularization CURE for short. The key step of CURE is to solve a biharmonic equation on a manifold. We further introduce a weighted version of CURE, called WeCURE, in a similar manner as the weighted nonlocal Laplacian (WNLL) method \cite{shi2017weighted}. Numerical experiments for image inpainting and semi-supervised learning show that the proposed CURE and WeCURE significantly outperform LDMM and WNLL respectively.
\end{abstract}



\section{Introduction}

Missing data recovery is a fundamental problem in imaging science and data analysis. In many cases, it can be formulated as a function interpolation problem in multiple dimension spaces. Let $u: \mathbb{R}^d\to \mathbb{R}$ be an unknown function. We would like to acquire its values on a set of points $P = \{\bm{p}_1,\ldots,\bm{p}_n\}\subset\mathbb{R}^d$. However, due to practical limitations, we are only able to observe its values on a subset $S = \{\bm{s}_1,\ldots,\bm{s}_m\}\subset P$. The goal of missing data recovery is to reconstruct the missing values of $u$ based on the observed values in $S$. In this paper, we focus on two kinds of typical and important tasks of missing data recovery, i.e. semi-supervised learning and image inpainting, though it can be well applied to other related tasks as well.

Since the problem of missing data recovery is an under-determined inverse problem, we can only hope to recover the missing values of $u$ if we have certain prior knowledge on $u$, e.g. $u$ belonging to a certain function class or having certain mathematical or statistical properties. Successful models include Rudin--Osher--Fatemi(ROF) model \cite{rudin1992nonlinear} and its variants \cite{gilboa2008nonlocal,bredies2010total,chan2000high}, the applied harmonic analysis models such as wavelets \cite{stephane1999wavelet,daubechies1992ten},  curvelet \cite{starck2002curvelet}, shearlet \cite{easley2008sparse, lim2010discrete} and wavelet frame \cite{bao2016image,cai2010simultaneous,chan2003wavelet, cai2009split,zhang2013l0,dong2012mra}, the Bayesian statistics based methods \cite{roth2005fields,shan2008high,zhu1997prior}; and the list goes on. 

More recently, people started to use low dimensional manifolds to describe the underlying relationship between the data points which serves as an effective geometric prior on the interpolant. For example, \cite{osher2016a, peyre2009manifold} observed that image patches, regarded as data points in a high dimension space, often lie on a low dimensional manifold; and \cite{coifman2006diffusion,zhu2018ldmnet} allowed the data lie close to (but may not be on) a certain low dimensional manifold. 

To harvest the low dimensional property of data, \cite{osher2016a} applied the following Dirichlet energy \cite{zhu2003a} to regularize the dimension of the embedded manifold $\mathcal{M}$
\begin{equation}\label{e:ldmm}
\text{LDMM}(u)= \frac{1}{2} \left\|\nabla_{\mathcal{M}}u\right\|_{L^2(\mathcal{M})}^2.
\end{equation}
In \cite{osher2016a}, the authors gave a geometric interpretation of the Dirichlet regularizer. \textcolor{blue}{They showed that the dimension of a smooth manifold embedded in $\mathbb{R}^d$ can be calculated by  a simple formula
$$
    dim(\mathcal{M})(\bm{x})=\sum_{j=1}^{d}|\nabla_{\mathcal{M}}\alpha_i(\bm{x})|^2
$$
where $\alpha_i$ is the coordinate function, for any $\bm{x}=(x_1,\cdots,x_d)\in \mathcal{M}\subset \mathbb{R}^d$, 
$\alpha_i(\bm{x})=x_i$}\\
This means that we can minimize the Dirichlet energy to enforce a penalty on the (local) dimensions of the underlying manifold. As a result, the authors referred to their method as the low dimensional manifold model (LDMM). To recover missing data, they proposed to minimize the Dirichlet energy subject to the constraints $u(\bm{s})=g(\bm{s})$, $\forall \mathbf{s}\in S$, where $g: S\to\mathbb{R}$ denotes the observed part of the underlying function $u$.

\subsection{Higher Order Regularization}

Only low dimension structure of the manifold does not readily ensure smoothness of the reconstructed manifold which may lead to unsatisfactory results \cite{nadler2009semi,el2016asymptotic,calder2018properly}. As a simple demonstration, we show in Figure \ref{fig:manifold} a degenerated interpolation result from the two data points labeled in red. Although the interpolated surface is also a low dimensional manifold, it is certainly not a smooth interpolation.

In this paper, we look for the proper interpolation by not only assuming low dimensionality of the manifold, but also the smoothness. For that, in addition to the Dirichlet energy, we further introduce a \textbf{CU}rvature \textbf{RE}gularization (\textbf{CURE}) term via biharmonic operator. The proposed CURE energy reads as follows
$$
\text{CURE}(u)=\text{LDMM}(u)+ \frac{\lambda}{2}\int_{\mathcal{M}}\textcolor{blue}{(\Delta_{\mathcal M} u)^2},
$$
where $\text{LDMM}$ is given by \eqref{e:ldmm}. Note that regularizing the curvature by introducing higher order energy term has already been proposed in image processing \cite{shen2003euler}. However, to the best of our knowledge, we are the first to promote curvature-like regularization for nonlocal image processing. Furthermore, inspired by the weighted nonlocal Laplacian (WNLL) method proposed by \cite{shi2017weighted} which can preserve the symmetry of the Laplace operator, we propose a weighted CURE (WeCURE) model which can significantly improve the results over the CURE model. To demonstrate the effectiveness of CURE and WeCURE, we test our model on semi-supervised learning and image inpainting task. Numerical results show that CURE/WeCURE produces significantly better results than LDMM/WNLL in both tasks. A glimpse of the results for image inpainting is shown in Figure \ref{fig:exp} where we can see the significant improvement of CURE over LDMM and WeCURE over WNLL. More details and numerical results can be found in \cref{sec:ssl} and \cref{sec:ip}.

\begin{figure}[H]
\centering
\includegraphics[scale=0.2]{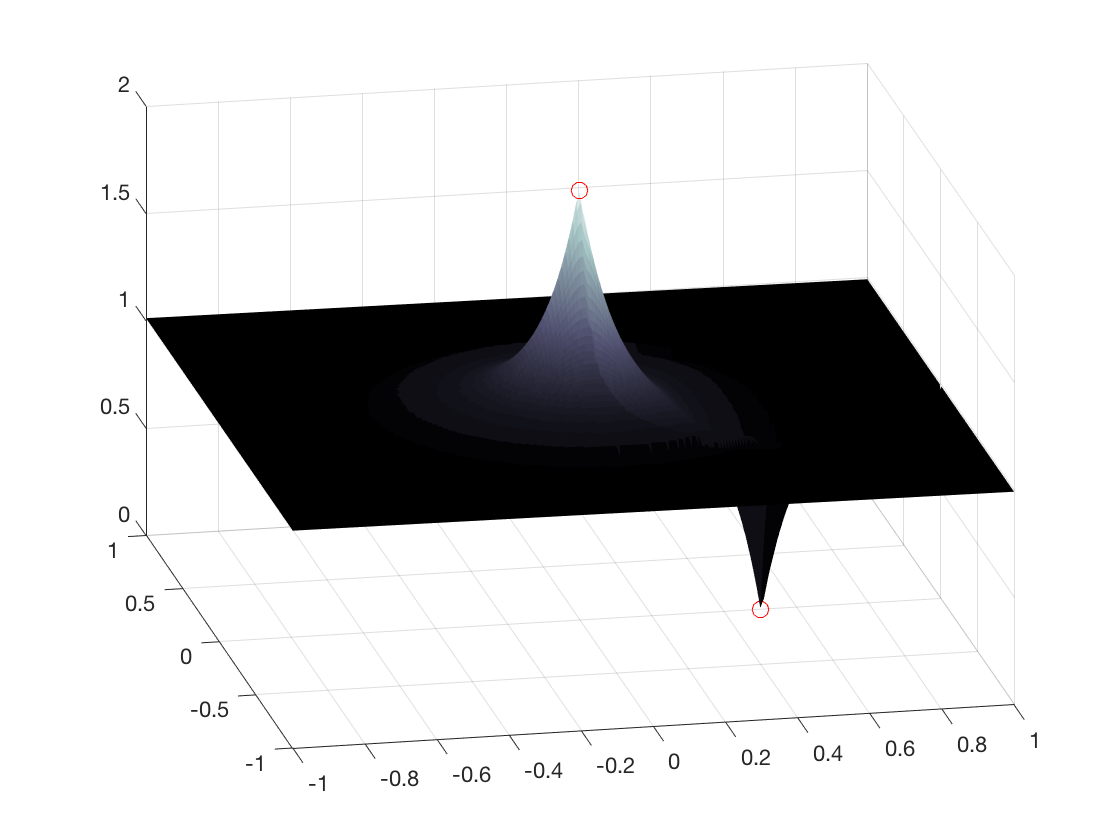}

\caption{\scriptsize{A low dimensional manifold without curvature regularization.}}
\label{fig:manifold}
\end{figure}

\subsection{Other Related Works}

Nonlocal patch-based image restoration methods\cite{dabov2006image,danielyan2012bm3d,buades2005non,buades2005a,gilboa2008nonlocal} have achieved great success in the literature. In addition, \cite{gilboa2007a,bertozzi2012diffuse,dong2017sparse} also introduced different graph Laplacian-based regularization on manifold and graphs. Our method, however, focuses on both smoothness and low dimensionality of the underlying data manifold. The most similar work to ours is \cite{agarwal2006higher}, where the authors also introduced a higher order regularization for semi-supervised learning. The difference is threefold. First, we extend the task to image inpainting rather than just semi-supervised learning. Secondly, we introduce a curvature perspective on the higher order regularization. Last but not least, the newly proposed weighted version of CURE, i.e. WeCURE, has significant performance boost in both image inpainting and semi-supervised learning.

Another approach to regularize the dimension of the manifold is through low-rank matrix completion \cite{gu2014weighted,lai2018manifold}. The basic idea is to group the patches by similarity and penalized the rank/nuclear norm of the matrix obtained by reshaping the stack of the similar patches. The work in this paper reveals a benefit of PDE-based approaches that higher order information, such as curvature, can be naturally incorporated in the model.

\subsection{Organization of the Paper}

The paper is organized as follows. The proposed CURE and WeCURE model are introduced in \cref{sec:wbih}, Numerical comparisons of CURE and WeCURE with LDMM and WNLL for semi-supervised learning and image inpainting are presented in \cref{sec:ssl} and \cref{sec:ip} respectively. The general setting of the asymptotic analysis of the proposed model is presented in \cref{sec:ana} and the complete proof is given in \cref{prof1,prof2}. Conclusions and summary are given in \cref{sec:conclusions}.

\begin{figure}[H]
\begin{minipage}[t]{0.28\linewidth}
\centering
\includegraphics[scale=0.2]{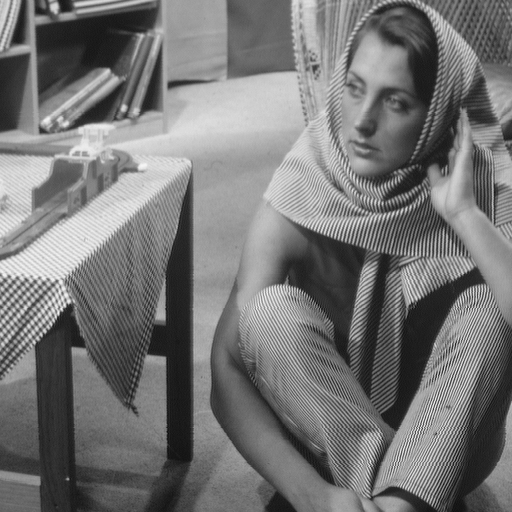}
\end{minipage}
\begin{minipage}[t]{0.28\linewidth}
\centering
\includegraphics[scale=0.2]{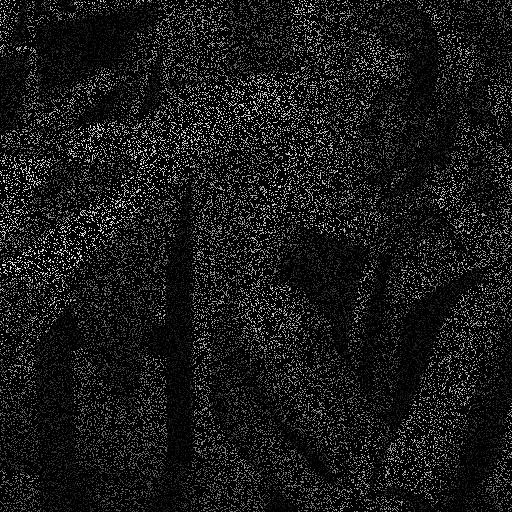}
\end{minipage}
\hspace{3mm}
\begin{minipage}[t]{0.35\linewidth}
\centering
\includegraphics[scale=0.266]{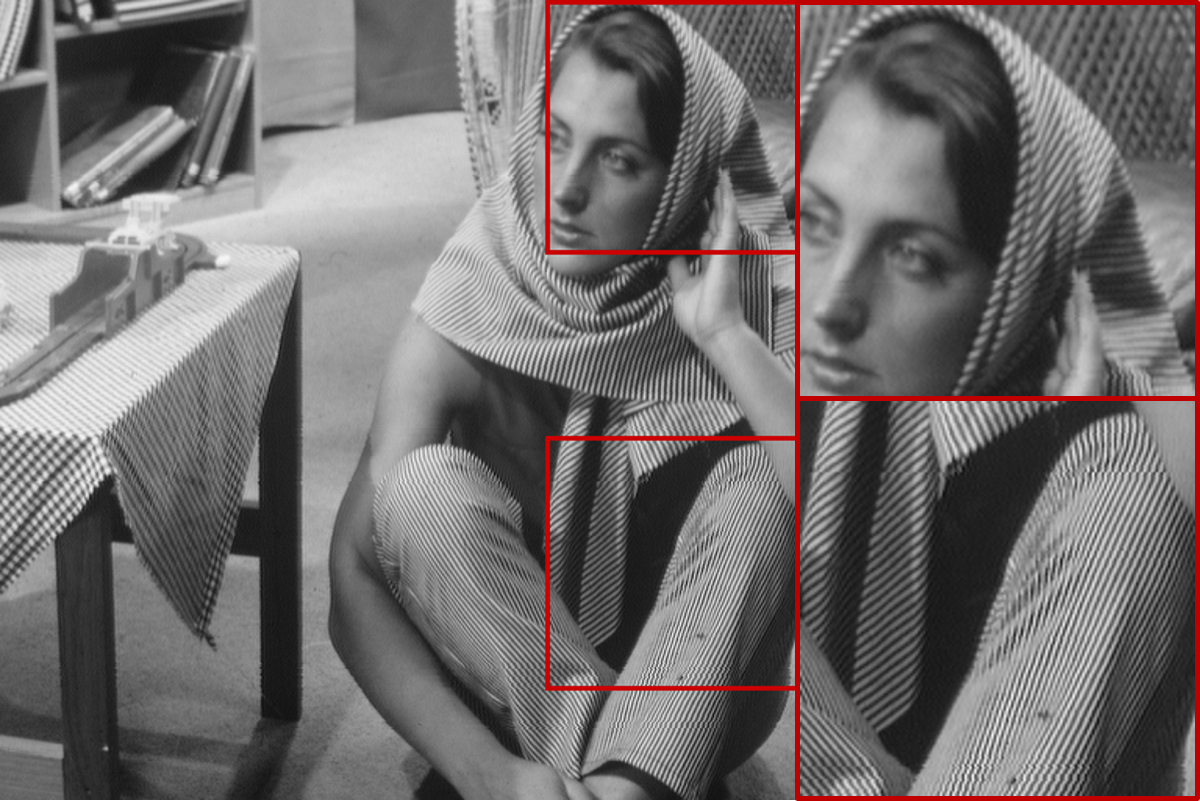}
\end{minipage}

\end{figure}
\begin{figure}[H]
\begin{minipage}[t]{0.24\linewidth}
\centering
\includegraphics[scale=0.16]{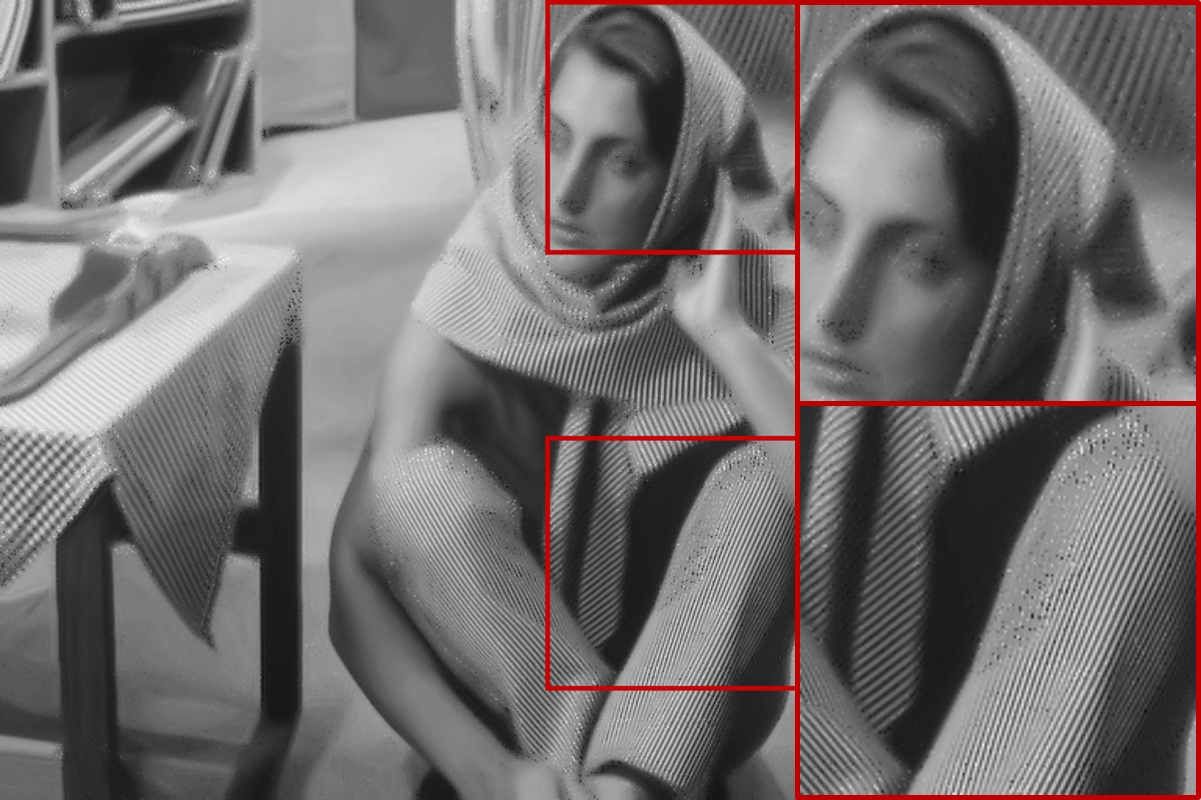}
\caption*{LDMM:PSNR=26.81dB, SSIM=0.68}
\end{minipage}
\begin{minipage}[t]{0.24\linewidth}
\centering
\includegraphics[scale=0.16]{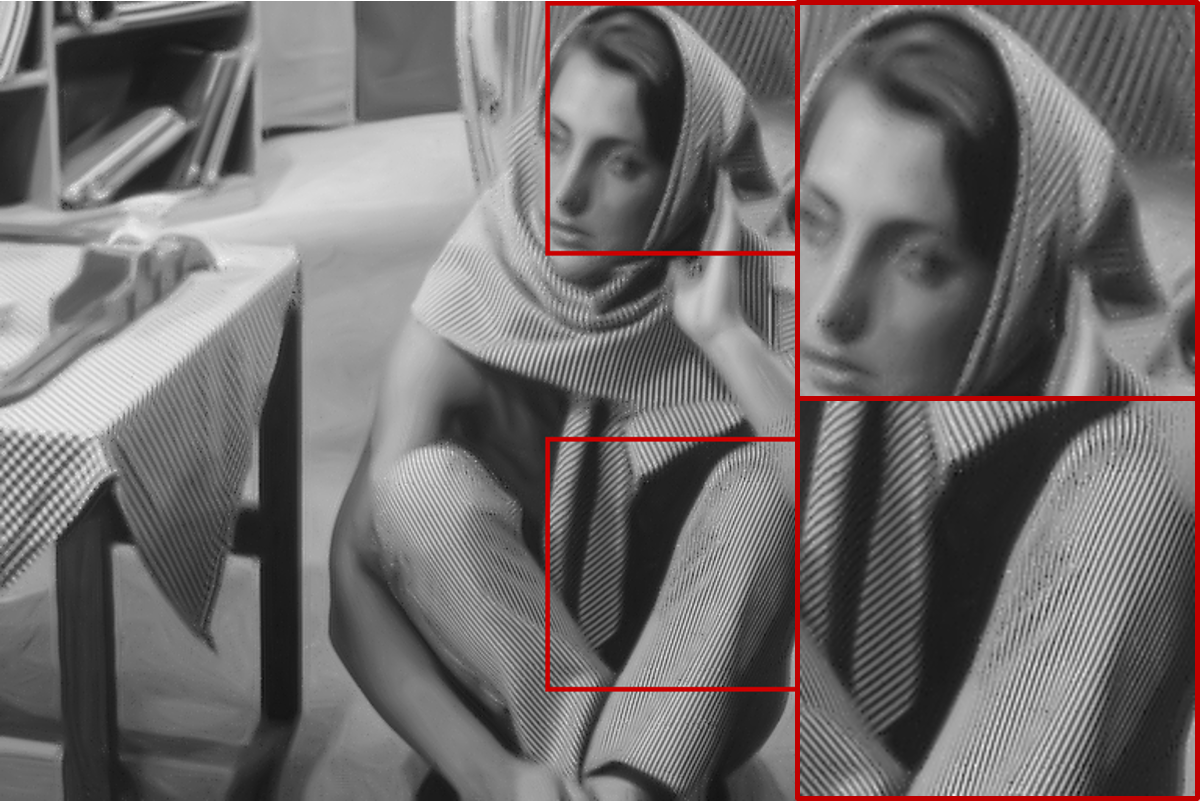}
\caption*{WNLL:PSNR=28.73dB, SSIM=0.73}
\end{minipage}
\begin{minipage}[t]{0.24\linewidth}
\centering
\includegraphics[scale=0.16]{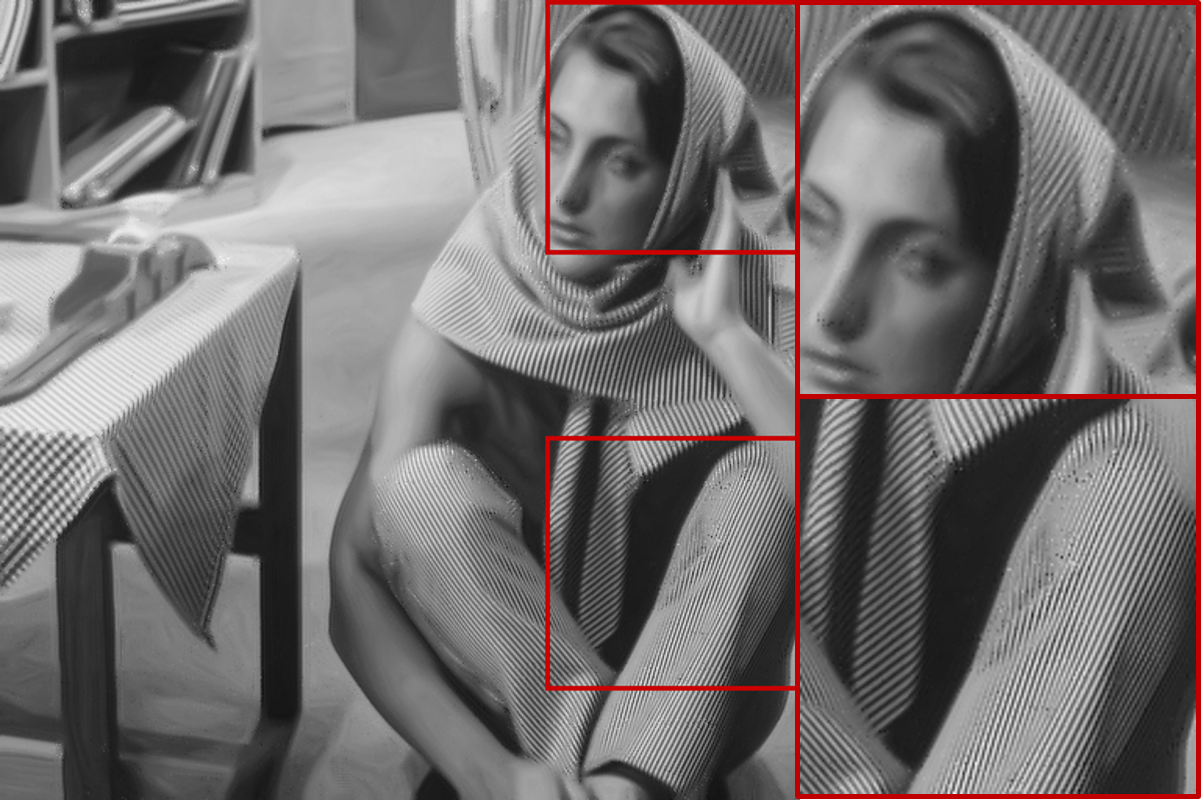}
\caption*{CURE:PSNR=28.97dB, SSIM=0.75}
\end{minipage}
\begin{minipage}[t]{0.24\linewidth}
\centering
\includegraphics[scale=0.16]{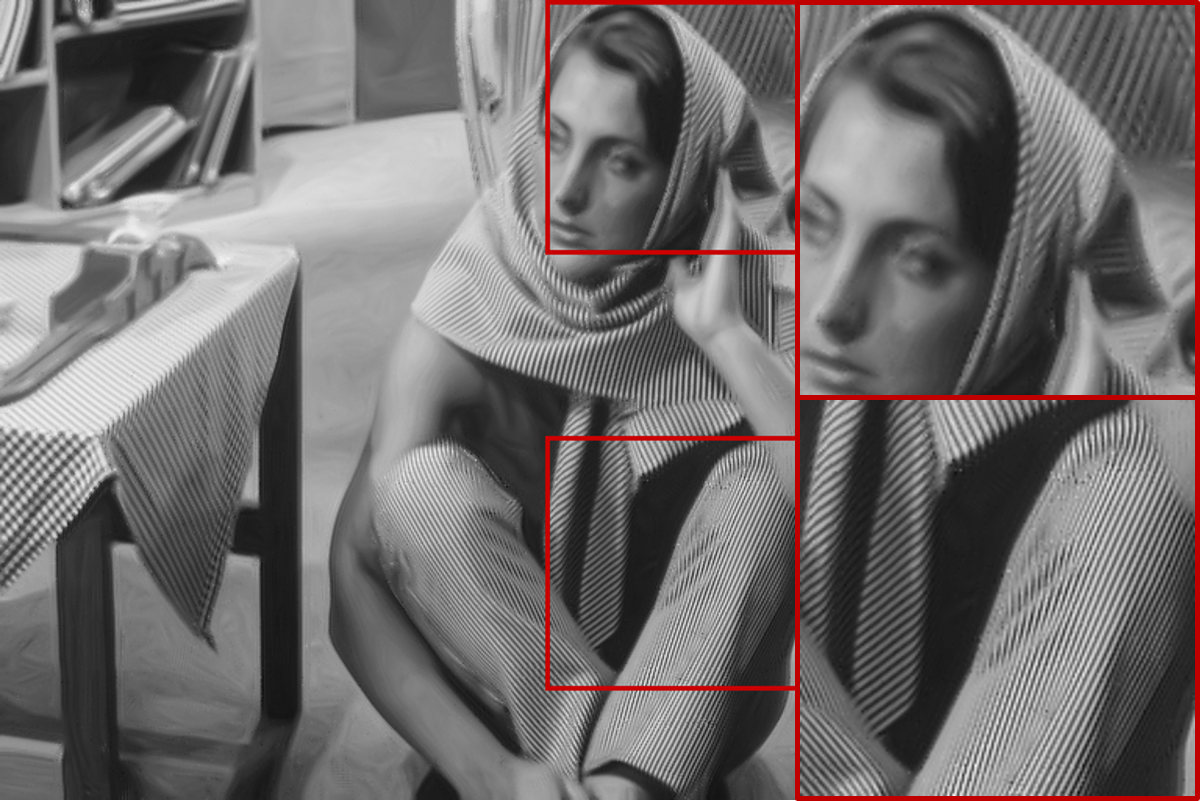}
\caption*{WeCURE:PSNR=29.78dB, SSIM=0.77}
\end{minipage}
\caption{\scriptsize{First row: original image, $20\%$ subsampled image, zoom-in views of the original image. Second row: inpainting results of LDMM, WNLL, CURE, WeCURE}}
\label{fig:exp}
\end{figure}

\section{Curvature Regularization (CURE): Model and General Algorithm}
\label{sec:wbih}

In this section, we first propose the CURE model and a weighted version of CURE. Then, we will discuss how (We)CURE can be applied to missing data recovery in general.

\subsection{CURE}

Let $\mathcal{M}$ be a smooth manifold embedded in $\mathbb{R}^d$ and locally parameterized as
$$
x=\psi(\alpha):U\subset \mathbb{R}^k\rightarrow\mathcal{M}\subset \mathbb{R}^d
$$
where $k=dim_x(\mathcal{M})$ is the local dimension of $\mathcal{M}$ at $x$, $\alpha=(\alpha^1,\ldots,\alpha^k)^\top\in\mathbb{R}^k$ and $x=(x_1,\ldots,x_d)^\top\in\mathcal{M}$. Let $\bm{u}=(u_1,u_2,\cdots,u_d)$ be the coordinate function on $\mathcal{M}$, \emph{i.e.} for $x\in\mathcal{M}$
$$
u_i(x)=x_i,\ 1\le i\le d.
$$

To enforce smoothness of the underlying manifold, we further regularize the curvature of the manifold. Recall that the mean curvature of a manifold $\mathcal{M}$ is defined as the trace of the second fundamental form \cite{manfio2015minimal}, \emph{i.e.} $H\vec{n}=g^{i,j}\nabla_i\nabla_j X$. Here $g^{i,j}$ is the metric tensor defined by $g_{i'j'}=\left<\partial_{i'},\partial_{j'}\right>=\sum_{l=1}^k \partial_{i'}\psi^{l}\partial_{j'}\psi^{l}$. If the coordinate function $\bm{u}(x)$ is an isometric immersion, the mean curvature can be calculated as $\|\Delta \bm{u}\|_2/k$, where $\Delta \bm{u}=(\Delta u_1,\Delta u_2,\cdots,\Delta u_d)$(see \cite{manfio2015minimal} for detail). 

Now, we are ready to introduce the CURE energy in continuum setting:
$$
\text{CURE}(u)=\text{LDMM}(u)+\frac{\lambda}{2}\int_{\mathcal{M}}\textcolor{blue}{(\Delta_{\mathcal M} u)^2},
$$
where $\text{LDMM}(u)$ is given by \eqref{e:ldmm}. The gradient $\nabla_{\mathcal{M}}u$ is commonly approximated by the nonlocal gradient in the discrete setting
$$
\nabla_{\mathcal{M}}u(\bm{x},\bm{y})\approx\sqrt{\omega(\bm{x},\bm{y})}(u(\bm{y})-u(\bm{x}))=:\nabla_P u(\bm{x},\bm{y}),\ \bm{x},\bm{y}\in P\subset\mathcal{M},
$$
where $P$ is a set with $n$ points on the manifold $\mathcal{M}$. Then, 
\begin{equation*}
\text{LDMM}(u)\approx\frac{1}{2}\sum_{\bm{x},\bm{y}\in P}{w(\bm{x},\bm{y})(u(\bm{x})-u(\bm{y}))^2}=\|\nabla_P u\|_2^2.
\end{equation*}
Here, $w(\bm{x},\bm{y})$ is a given symmetric weight function which is often chosen to be a Gaussian weight $w(\bm{x},\bm{y})$=exp$(-\frac{\left\|\bm{x}-\bm{y}\right\|^2}{\sigma^2})$, where $\sigma$ is a parameter and $\left\|\cdot \right\|$ denotes the Euclidean norm in $\mathbb{R}^{\frac{n(n-1)}{2}}$. The negative of the first variation of $\|\nabla_P u\|_2^2$ takes the form
\begin{equation*}
-\partial_u\left(\|\nabla_P u\|_2^2\right)=\sum_{\bm{y}\in P}{w(\bm{x},\bm{y})(u(\bm{x})-u(\bm{y}))},
\end{equation*}
which is the nonlocal Laplacian that has been used in image processing \cite{buades2006a,buades2005a,gilboa2007a,gilboa2008nonlocal}. It is also called graph Laplacian in spectral graph and machine learning literature \cite{chung1997a,zhu2003a}. To simplify the notation, we use $GL$ to denote the graph Laplacian \cite{li2016a,trillos2016continuum,trillos2018variational}: $$GLu(x):=\sum_{\bm{y}\in P}{w(\bm{x},\bm{y})(u(\bm{x})-u(\bm{y}))}.$$

Now, the proposed \textbf{CURE} model can be cast as the following optimization problem in the discrete setting
\begin{equation}\label{model:cure}
    \min_u\ \|\nabla_P u\|_2^2+\frac{\lambda}{2}\|GLu\|_2^2.
\end{equation}

In \cite{shi2017weighted}, a weighted nonlocal Laplacian (WNLL) method was introduced to balance the loss at both labeled and unlabeled points and to preserve the symmetry of the Laplace operator at the same time. Let $S\subset P$ be a set with labeled points. The WNLL model in the discrete setting is given by
\begin{equation*}
 \text{WNLL}(u)=\|\left(\nabla_P u\right)_{|P\backslash S}\|_2^2 + \frac{|P|}{|S|}\|\left(\nabla_P u\right)_{|S}\|_2^2,
\end{equation*}
where $$\|\left(\nabla_P u\right)_{|S}\|_2^2:=\sum_{\bm{x}\in S,\bm{y}\in P}\frac{1}{2}w(\bm{x},\bm{y})(u(\bm{x})-u(\bm{y}))^2,$$ and similarly for $\|\left(\nabla_P u\right)_{|P\backslash S}\|_2^2$.

Following a similar idea as that in WNLL, we propose the weighted CURE model (\textbf{WeCURE}) in the discrete setting
\begin{equation}\label{model:wecure}
\min_u\ \text{WeCURE}(u):=\text{WNLL}(u)+\lambda\left[\|\left(GLu\right)_{|P\backslash S}\|_2^2 + \frac{|P|}{|S|}\|\left(GLu\right)_{|S}\|_2^2\right],
\end{equation}
where $$\|\left(GLu\right)_{|S}\|_2^2=\sum_{\bm{x}\in S}\left(\sum_{ \bm{y}\in P}{w(\bm{x},\bm{y})(u(\bm{x})-u(\bm{y}))}\right)^2$$ and similarly for $\|\left(GLu\right)_{|P\backslash S}\|_2^2$.

\subsection{CURE for Missing Data Recovery}

For missing data recovery, we can simply minimize the CURE or WeCURE energy with respect to the constraints $u(\bm{x})=g(\bm{x}),\bm{x}\in S$ where $g$ is the observed values of the underlying function to be recovered. We discuss implementation details for WeCURE. CURE is a special case of WeCURE with all weights equal to 1.

Recall the definition of the energy function of WeCURE \eqref{model:wecure} and notice that $u(\bm{x})=g(\bm{x}), \bm{x}\in S$. Then, WeCURE model for missing data recovery can be rewritten as
\begin{equation}\label{model:wecure:mdr}
\min_{u|_{P\backslash S}} \text{WNLL}\left(\left[\begin{array}{cc}u|_{P\backslash S}\\
g
\end{array} \right]\right)+ \lambda\left\|\sqrt{D}\cdot GL\left[\begin{array}{cc}u|_{P\backslash S}\\
g
\end{array} \right]\right\|_2^2,
\end{equation}
where $D=\diag\{d_1,d_2,\ldots,d_{|P|}\}$ with $d_i=1$ for $\bm{x}_{i}\in P\backslash S$ and $d_i=\frac{|P|}{|S|}$ for $\bm{x}_{i}\in S$, and $GL$ is the $|P|\times|P|$ matrix of graph Laplacian. The first variation of \eqref{model:wecure:mdr} is
$$
\partial_{u|_{P\backslash S}} \text{WeCURE}\left(\left[\begin{array}{cc}u|_{P\backslash S}\\
g
\end{array} \right]\right)=\partial_{u|_{P\backslash S}} \text{WNLL}\left(\left[\begin{array}{cc}u|_{P\backslash S}\\
g
\end{array} \right]\right)+\lambda\partial_{u|_{P\backslash S}} \left\|\sqrt{D}\cdot GL\left[\begin{array}{cc}u|_{P\backslash S}\\
g
\end{array} \right]\right\|_2^2.
$$
Note that $$\left\|\sqrt{D}\cdot GL\left[\begin{array}{cc}u|_{P\backslash S}\\
g
\end{array} \right]\right\|_2^2=\left\|\sqrt{D}\cdot GL\left[\begin{array}{cc}u|_{P\backslash S}\\
0
\end{array} \right]+\sqrt{D}\cdot GL\left[\begin{array}{cc}0\\
g
\end{array} \right]\right\|_2^2.$$ 
Thus
\begin{align*}
\partial_{u|_{P\backslash S}} \text{WeCURE}\left(\left[\begin{array}{cc}u|_{P\backslash S}\\
g
\end{array} \right]\right)
=&\partial_{u|_{P\backslash S}} \text{WNLL}\left(\left[\begin{array}{cc}u|_{P\backslash S}\\
g
\end{array} \right]\right)\\
&+\lambda GL^T\cdot D\cdot GL\left[\begin{array}{cc}u|_{P\backslash S}\\
0
\end{array} \right]+\lambda GL^T\cdot D\cdot GL\left[\begin{array}{cc}0\\
g
\end{array} \right].
\end{align*}
Then, the solution to problem \eqref{model:wecure:mdr} can be given by solving the following Euler-Lagrange equation
\begin{equation}\label{model:wecure:mdr:el}
\begin{array}{ll}
\left(GL\cdot \left[\begin{array}{cc}u|_{P\backslash S}\\
0
\end{array} \right]+\gamma \cdot DW\cdot \left[\begin{array}{cc}u|_{P\backslash S}\\
0
\end{array} \right]
+\lambda GL^T\cdot D\cdot GL\cdot \left[\begin{array}{cc}u|_{P\backslash S}\\
0
\end{array} \right]\right)({\bm{x}})
& \\
\hspace*{0.3in}=\sum_{\bm{y}\in S} w(\bm{x},\bm{y})g(\bm{y})+\gamma \sum_{\bm{y}\in S} w(\bm{y},\bm{x})g(\bm{y})-\lambda\left(GL^T\cdot D\cdot GL \left[\begin{array}{cc}0\\
g
\end{array} \right]\right)({\bm{x}}), & \bm{x}\in P\backslash S,
\end{array}
\end{equation}
where $DW=\text{diag}(w_1,w_2, \ldots,w_{|P|})$ with $w_i=\sum_{\bm{y}\in S} w(\bm{x}_i,\bm{y})$ and $\gamma$ is the weighted coefficient in WNLL. The above linear system is symmetric positive definite and sparse which can be solved efficiently by iterative solvers such as the conjugate gradient method. We remark that, for (non-weighted) CURE method, we only need to replace matrix $D$ above by identity matrix $Id_{|P|\times |P|}$. We summarize (We)CURE algorithm for missing data recovery in Algorithm \ref{alg:cure}.

\begin{algorithm}[htb]
\caption{(We)CURE for Missing Data Recovery}
\label{alg:cure}
\begin{algorithmic}
\Require \small Given point set $P = \{\bm{p}_1,\ldots,\bm{p}_n\} \subset \mathbb{R}^{d}$ and a partially labeled set $S\subset P$, and given the function values of $u$ on $S$, \emph{i.e.} $u(\bm{x})=g(\bm{x})$ for $\bm{x}\in S$.
\Ensure \small A recovered function $u$ on $P$.
\State Calculate the weight matrix $W=(w(\bm{p}_i,\bm{p}_j))_{n\times n}$ and the graph Laplacian $GL$. Set $DW=\text{diag}([\sum_{j=1}^m w(\bm{p}_i,\bm{p}_j)]_{i=m+1:n})$.
\State Solving the linear system \eqref{model:wecure:mdr:el} for $u|_{P\backslash S}$.
\end{algorithmic}
\end{algorithm}

\section{CURE for Semi-Supervised Learning}
\label{sec:ssl}

Semi-supervised learning is a challenging and yet frequently encountered machine learning task. It can be formulated as a missing data recovery problem \cite{zhu2003a}. Given a data set $P = \{\bm{p}_1,\ldots,\bm{p}_n\} \subset \mathbb{R}^{d}$, we assume there are totally $l$ different classes. Let $S\subset P$ be a subset of $P$ with labels, i.e
\begin{displaymath}
S=\bigcup_{i=1}^{l}S_i,
\end{displaymath}
where $S_{i}\subset P$ is the subset with label $i$. It is typical for semi-supervised learning that $|S|$ is far less than $|P|$. The objective of 
semi-supervised learning is to extend labels to the entire data set $P$. Our algorithm is summarized in Algorithm \ref{alg:SSL}.

\begin{figure}
\centering
\includegraphics[scale=0.2]{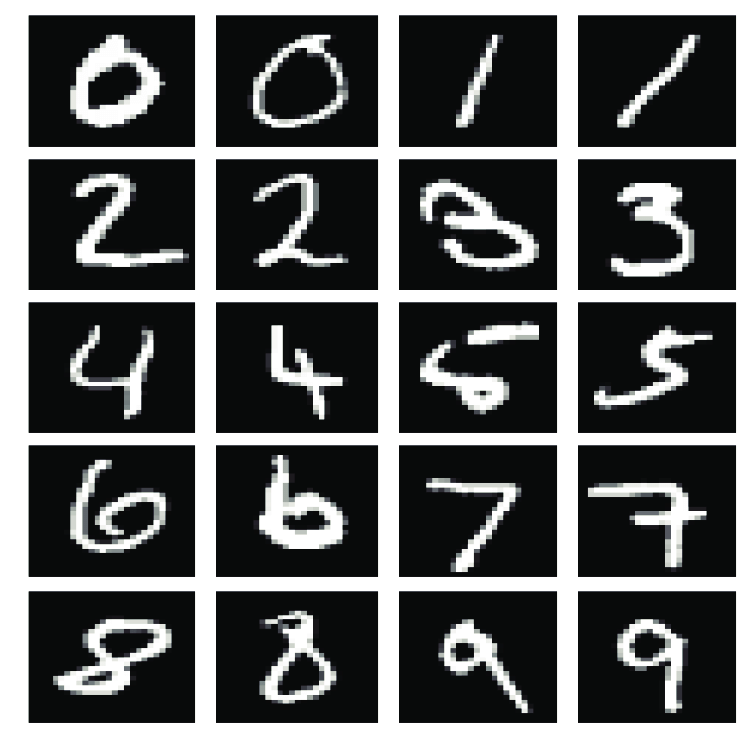}
\includegraphics[scale=0.4]{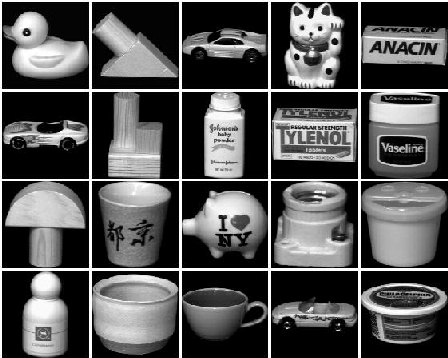}
\caption{\scriptsize{ Some images in MNIST and COIL20 dataset.}}\label{fig:mnist}
\end{figure} 

\begin{algorithm}[htb]
\caption{(We)CURE for Semi-Supervised Learning}
\label{alg:SSL}
\begin{algorithmic}
\Require \small Point set $P = \{\bm{p}_1,\ldots,\bm{p}_n\} \subset \mathbb{R}^{d}$ and a partially labeled set $S=\bigcup_{i=1}^{l}S_i$.
\Ensure \small A complete label assignment $L:P\rightarrow \{1,2,\ldots,l\}$
\State \textbf{for} $i = 1:l$ \quad \textbf{do}\\
\indent \indent \small Compute $\phi_{i}$ on $P$ with the known observation
\begin{displaymath}
\phi_i(\bm{x})=1, \bm{x}\in S_i, \qquad \phi_i(\bm{x})=0, \bm{x}\in S\backslash S_i,
\end{displaymath}
by Algorithm \ref{alg:cure}.
\State \textbf{end for}
\State \textbf{for} $\bm{x}\in P\backslash S$ \quad \textbf{do}\\
\indent \indent \small Label $\bm{x}$ as following
\begin{displaymath}
L(\bm{x})=k,\qquad {\rm where }\quad k={\rm arg}\max_{1\leq i\leq l}{\phi_i(\bm{x})}
\end{displaymath}
\State \textbf{end for}
\end{algorithmic}
\end{algorithm}

We test WNLL, Weighted Nonlocal Total Variation (WNTV) \cite{Li2019weighted}, CURE, WeCURE on the MNIST dataset \cite{lecun1998mnist} of handwritten digits classification \cite{burges-a}, COIL20 dataset\cite{Nene96columbiaobject} of object classification and ISOLET dataset\cite{Dua:2019} of spoken letter recognition. Some sample images from MNIST and COIL20 are shown in Figure \ref{fig:mnist}. The MNIST dataset contains 70,000 gray-scale images of size 28 $\times$ 28 with 10 classes of digits going from 0 to 9. Each class contains 7,000 images. Each image can be seen as a point in a 784-dimension Euclidean space. The COIL20 dataset contains 20 objects, and each object has 72 images. The size of each image is 32 $\times$ 32 pixels, with 256 grey levels per pixel. Thus, each image is represented by a 1024-dimensional vector. The ISOLET dataset contains 150 subjects who spoke the name of each letter of the alphabet twice. The speakers are grouped into sets of 30 speakers each and are referred to as isolet1 through isolet5. In our experiment, we use isolet1 which consists of 1560 samples with each sample represented by a 617-dimensional vector.

The weight function $w(\bm{x},\bm{y})$ is constructed as
\begin{equation}\label{e:weightfunction}
w(\bm{x},\bm{y})={\rm exp}\left(-\frac{\left\|\bm{x}-\bm{y}\right\|^2}{\sigma(\bm{x})^2}\right),
\end{equation}
where $\sigma(\bm{x})$ is chosen to be the distance between $\bm{x}$ and its $k$th nearest neighbor ($k=20$ in MNIST, $k=15$ in COIL20 and ISOLET).
To make the weight matrix sparse, the weight $w(\bm{x},\bm{y})$ is truncated to the 50 nearest neighbors.\\

In our test on MNIST, we choose five different sampling rates to form the training set: labeling 700, 100, 70, 50 and 35 images in the whole dataset at random. For each sampling rate, we repeat the test results 10 times. In our test on COIL20 and ISOLET, we choose three different sampling rates to form the training set: labeling $2\%$, $5\%$, $10\%$ at random. For each sampling rate, we repeat the test 10 times. Figure \ref{fig:mnistssl} shows the success rate of WNLL, CURE, and WeCURE method on MNIST dataset. The first five images of Figure \ref{fig:mnistssl} show the success rate for each sampling rate, while the last image shows the average success rate for each of the five sampling rate. It can be clearly
observed that the proposed CURE and WeCURE outperform WNLL for all the tested cases. With a high sampling rate, WeCURE is comparable with CURE, whereas WeCURE outperforms CURE in the cases with lower sampling rates. In terms of average success rate, both CURE and WeCURE outperform WNLL. We also compare (We)CURE with WNLL and Weighted Nonlocal Total Variation (WNTV) \cite{Li2019weighted} in Table \ref{table:mnist}. It can be seen that (We)CURE significantly outperforms both WNLL and WNTV in cases with lower sample rates (50/70000,100/70000). Table \ref{table:coil20} shows the result on COIL20 and ISOLET dataset. It can be seen that WeCURE outperforms CURE and WNLL by $3\%\sim4\%$.

\begin{figure}[H]
\begin{minipage}[t]{0.32\linewidth}
\centering
\includegraphics[scale=0.12]{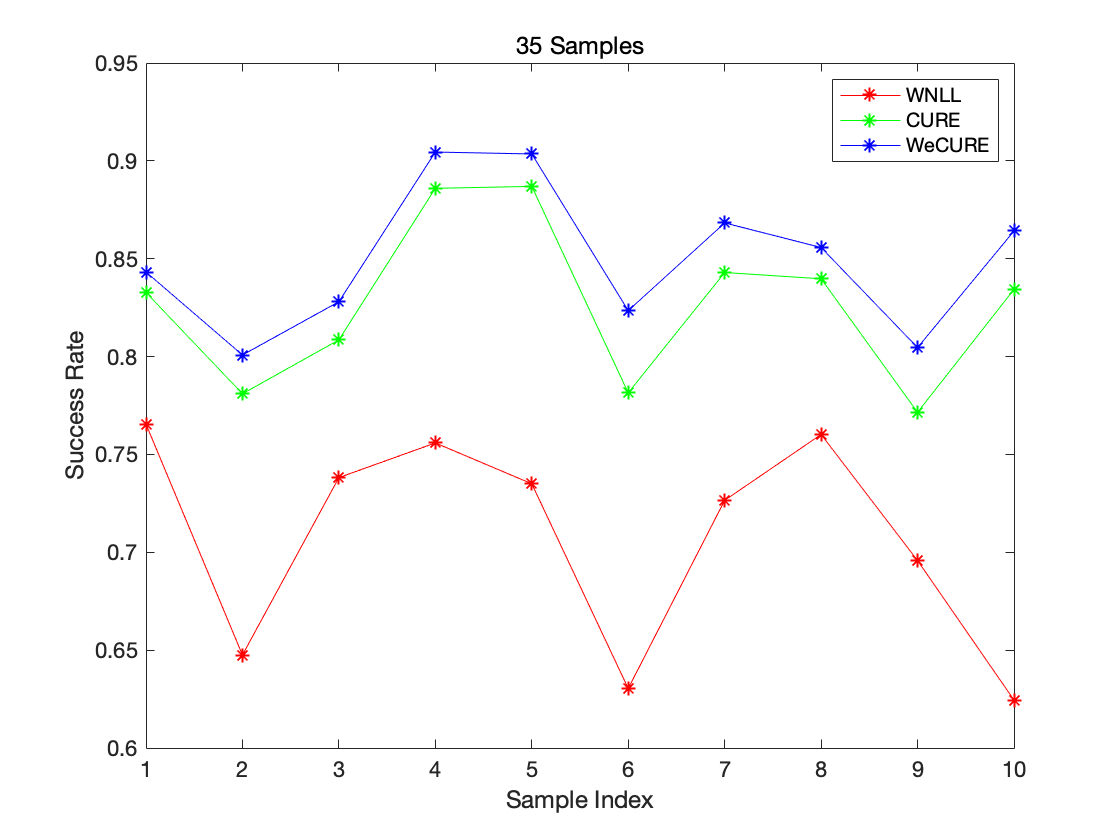}
\end{minipage}
\begin{minipage}[t]{0.32\linewidth}
\centering
\includegraphics[scale=0.12]{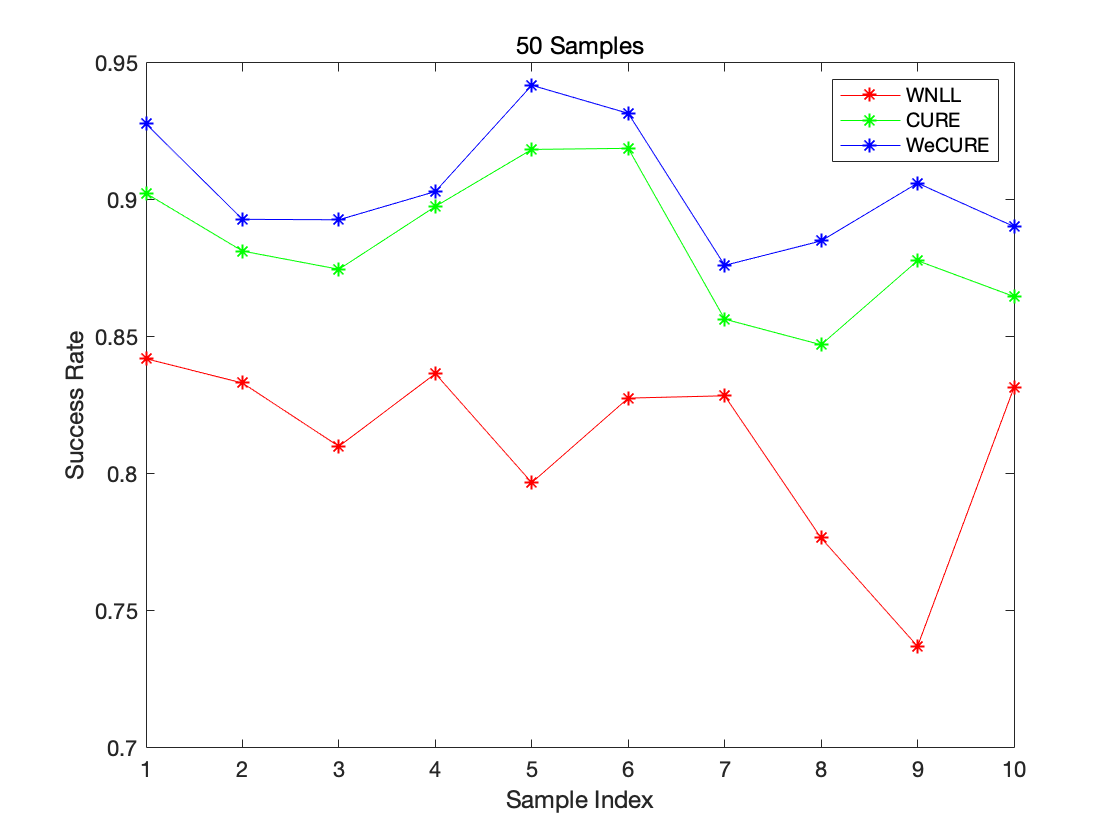}
\end{minipage}
\begin{minipage}[t]{0.32\linewidth}
\centering
\includegraphics[scale=0.12]{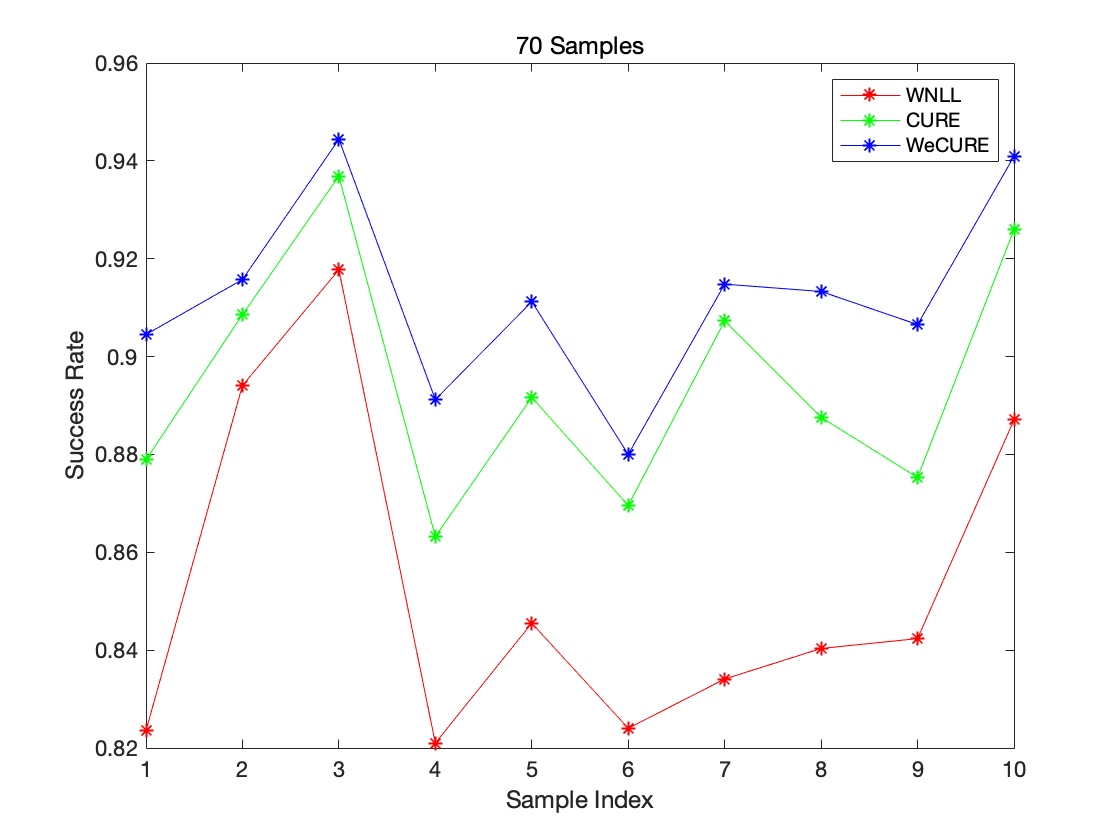}
\end{minipage}
\end{figure}
\begin{figure}[H]
\begin{minipage}[t]{0.32\linewidth}
\centering
\includegraphics[scale=0.12]{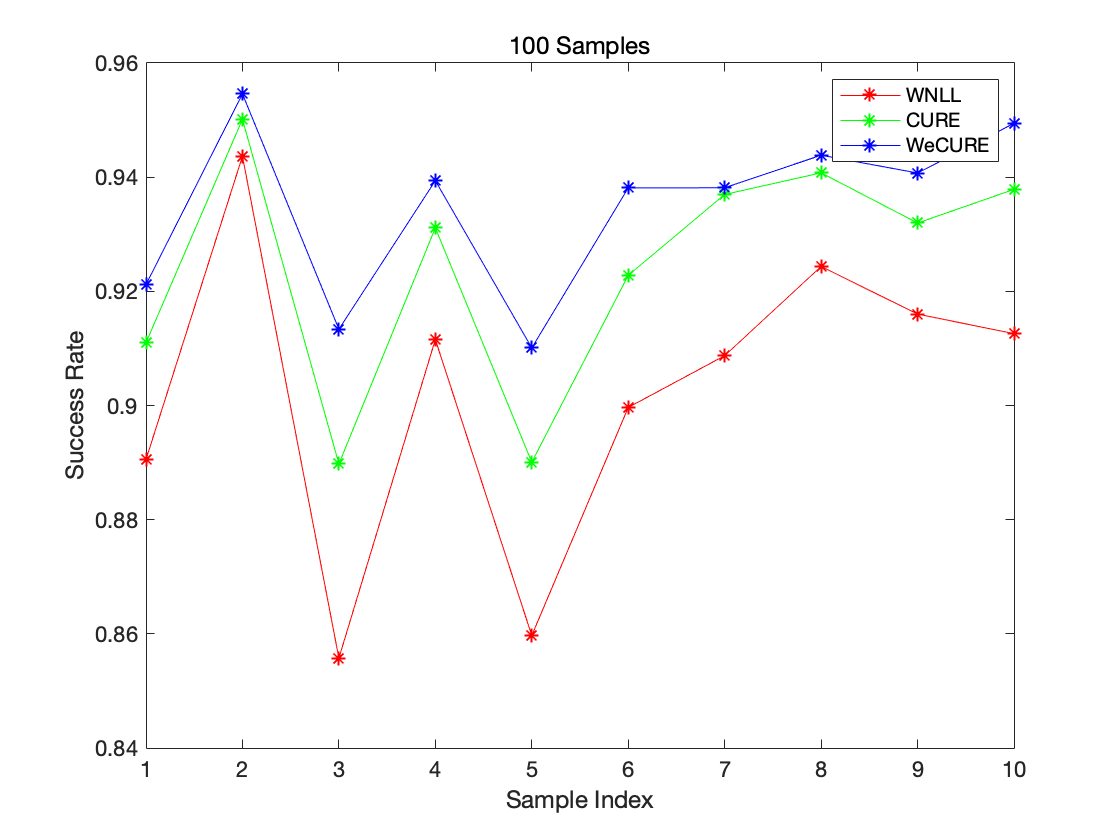}
\end{minipage}
\begin{minipage}[t]{0.32\linewidth}
\centering
\includegraphics[scale=0.12]{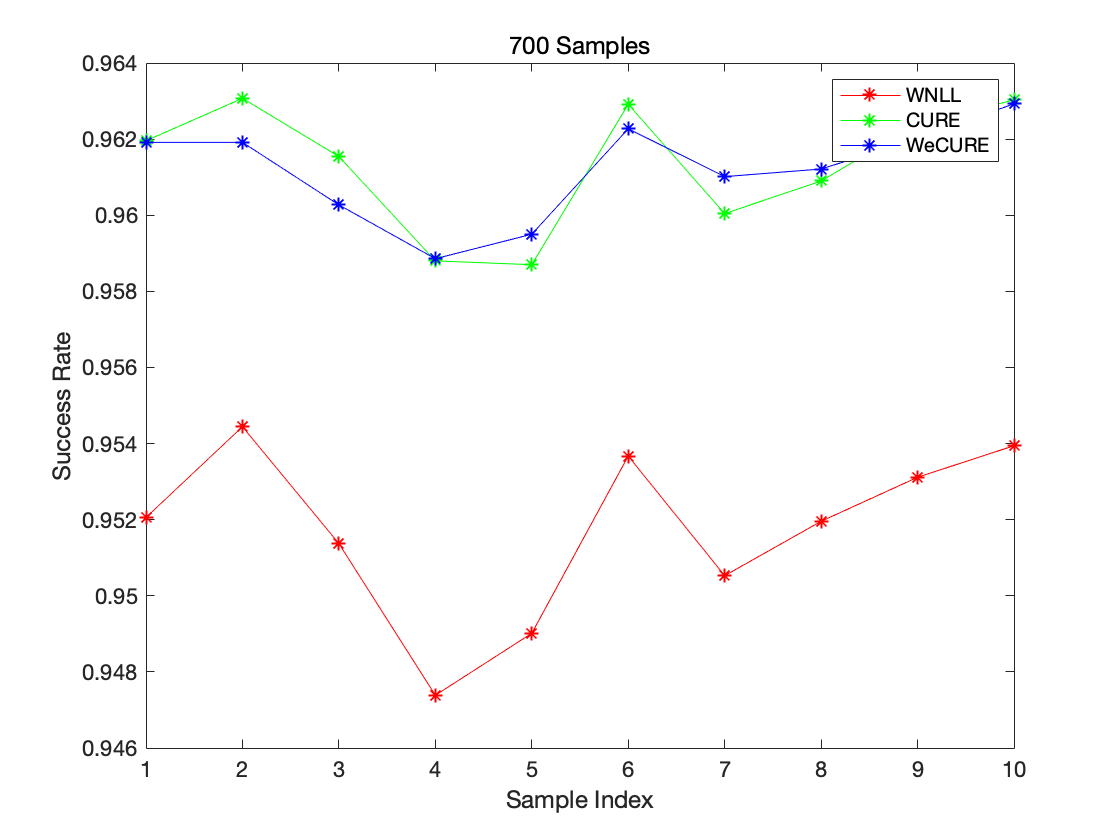}
\end{minipage}
\begin{minipage}[t]{0.32\linewidth}
\centering
\includegraphics[scale=0.12]{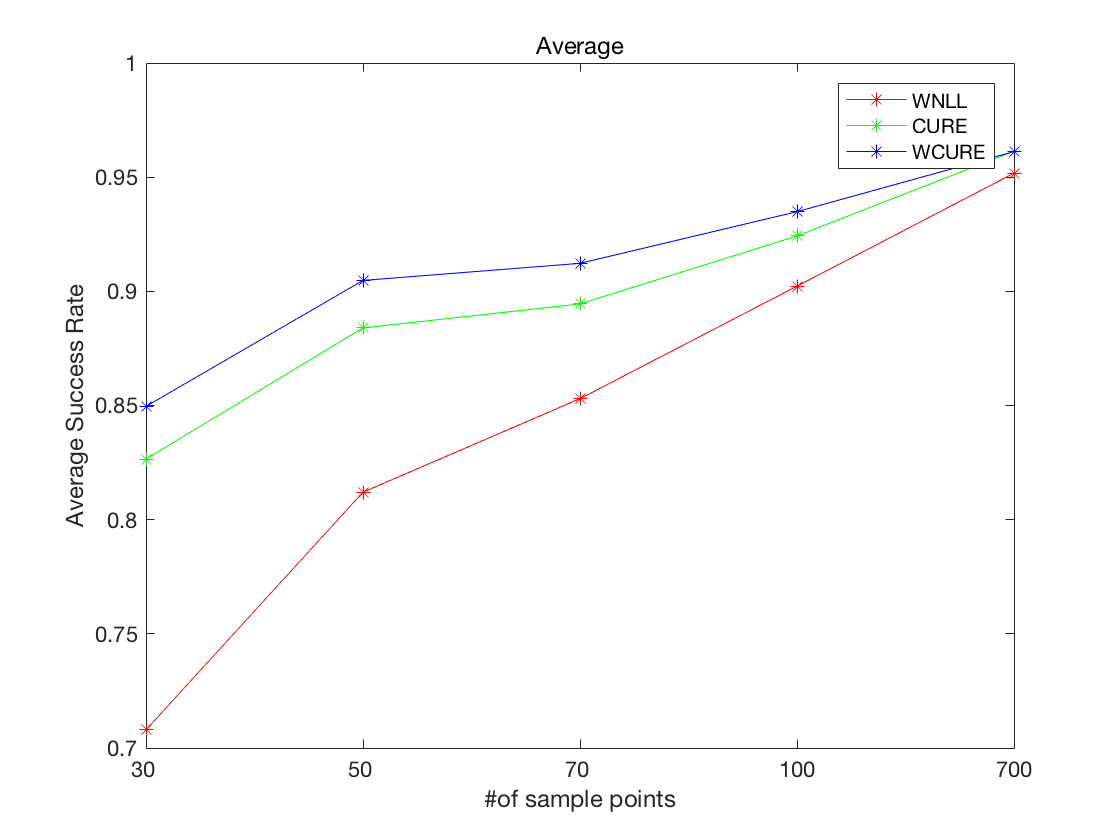}
\end{minipage}
\caption{\scriptsize{Comparisons of success rates by WNLL, CURE and WeCURE on MNIST.}}
\label{fig:mnistssl}
\end{figure}

\begin{table}[!hbp]
\centering
\begin{tabular}{|c|ccc|}
\hline
\textbf{Method}    & 50/70000 & 100/70000 & 700/70000 \\ \hline\hline
WNLL\cite{shi2017weighted}   & 73.60    & 87.84     & 93.25     \\
WNTV\cite{Li2019weighted}   & 78.35    & 89.86     & 94.08     \\
CURE   & \color[HTML]{0000FF} {88.40}    & \color[HTML]{0000FF} {92.42}     & \color[HTML]{FF0000} {96.13}     \\
WeCURE &\color[HTML]{FF0000} {90.48}    & \color[HTML]{FF0000} {93.49}     & \color[HTML]{0000FF} {96.12}    \\\hline
\end{tabular}
\caption{Classification accuracy in percentage for MNIST. The best results are in \textcolor{red}{red} and the second best results are in \textcolor{blue}{blue}.}
\label{table:mnist}
\end{table}

\begin{table}[!hbp]
\centering
\begin{tabular}{|c|ccc|ccc|}
\hline\hline
\multirow{2}*{Method}&\multicolumn{3}{|c|}{COIL20}&\multicolumn{3}{|c|}{ISOLET}\\
   & 2$\%$ & 5$\%$  & 10$\%$& 2$\%$ & 5$\%$  & 10$\%$  \\ \hline\hline
GL   & 55.61    & 68.50     & 76.11   & 31.19    & 45.51     & 66.27   \\
WNLL\cite{shi2017weighted}   & 59.59    & 74.13     & 80.65   & 49.12    & 61.90     & 73.05   \\
CURE   & \color[HTML]{0000FF} {59.73}    & \color[HTML]{0000FF} {74.77}     & \color[HTML]{0000FF} {80.91}   & \color[HTML]{0000FF} {49.14}    & \color[HTML]{0000FF} {61.94}     & \color[HTML]{0000FF} {73.23}  \\
WeCURE &\color[HTML]{FF0000} {63.29}    & \color[HTML]{FF0000} {77.65}     & \color[HTML]{FF0000} {84.76}  &\color[HTML]{FF0000} {52.65}    & \color[HTML]{FF0000} {64.92}     & \color[HTML]{FF0000} {76.50}    \\\hline
\end{tabular}
\caption{Classification accuracy in percentage for COIL20 and ISOLET. The best results are in \textcolor{red}{red} and the second best results are in \textcolor{blue}{blue}.}
\label{table:coil20}
\end{table}

\section{CURE for Image Inpainting}
\label{sec:ip}
\indent In this section, we apply (We)CURE to reconstruct the images with partially observed pixels. We adopt the assumption that image patches lie on a low dimensional and smooth manifold. Given an image $f \in \mathbb{R}^{m\times n}$, for any $(i,j)\in\{1,2,\ldots,m\}\times\{1,2,\ldots,n\}$, we
define an $s_1\times s_2$ image patch as $$p_{ij}(f)=\{f(\tilde i,\tilde j): i-(s_1-1)/2\le \tilde i\le i+(s_1-1)/2,\ j-(s_2-1)/2\le \tilde j\le j+(s_2-1)/2\},$$ where we assume $s_1$ and $s_2$ are odd integers and we adopt reflective boundary conditions for $(i,j)$ near image boundary. Define the patch set $P(f)$ as the collection of all patches:
\begin{displaymath}
P(f) = \{p_{ij}(f):(i,j)\in\{1,2,\ldots,m\}\times\{1,2,\times,n\}\}\subset \mathbb{R}^d,\quad d=s_1\cdot s_2.
\end{displaymath}
Define a function $u$ on $P(f)$ as 
\begin{displaymath}
u(p_{ij}(f)) = f(i,j),
\end{displaymath}
where $f(i,j)$ is the intensity of image $f$ at pixel $(i,j)$.

Now, suppose we only observe the image on a subset of pixels $\Omega \subset \{(i,j):1\leq i\leq m, 1\leq j\leq n\}$. We would like to recover the entire image $f$ from the observed data $f|_\Omega$. This problem can be recast as the interpolation of the function $u$
on the patch set $P(f)$ with $u$ being given in $S \subset P(f)$, $S = \{p_{ij}(f):(i,j)\in \Omega\}$. This falls into the general algorithmic framework of (We)CURE for missing data recovery (Algorithm \ref{alg:SSL}). Notice that the patch set $P(f)$ is unknown. Thus, we need to iterative update the patch set $P(f)$. We summarize the (We)CURE algorithm for this problem in Algorithm \ref{alg:SIR}.

\begin{algorithm}[htb]
\caption{Subsampled image restoration By WeCURE}
\label{alg:SIR}
\begin{algorithmic}
\Require \small A subsampled image $f|_{\Omega}$
\Ensure \small A recovered image $u$
\State \small Generate initial image $u^{0}$
\State \textbf{while} not converge \textbf{do}
\begin{algorithmic}[1]
\State  Generate the semi-local patch set $\Bar{P}(u^n)$ from current image $u^{n}$ and get corresponding labeled set $S^{n}\subset \Bar{P}(u^n)$
\State \small Update the image by computing $u^{n+1}$ on $P(u^n)$, with the known observation
\begin{displaymath}
u^{n+1}(\bm{x})=f(\bm{x}),\quad \bm{x}\in S^n.
\end{displaymath}
by Algorithm \ref{alg:cure}.
\State $n\leftarrow n+1.$
\end{algorithmic}
\textbf{end while}

\end{algorithmic}
$u=u^{n}$
\end{algorithm}

The weight function $w(\bm{x},\bm{y})$ is chosen as \eqref{e:weightfunction}. Here, $x,y\in \mathbb{R}^{d+2}$ are semi-local patches and $\sigma(\bm{x})$ is chosen to be the distance between $\bm{x}$ and its 20th nearest neighbor. To make the weight matrix sparse, the weight is truncated to the 50 nearest neighbors. In the semi-local patches, the local coordinate is normalized to have the same amplitude as the image intensity,
$$
(\Bar{P}\bm{u})(x)=[(P\bm{u})(x),\lambda \Bar{x}]
$$
with
$$
\Bar{x}=\left(\frac{x_1\|(f|_\Omega)\|_\infty}{m},\frac{x_2\|(f|_\Omega)\|_\infty}{n}\right),
$$
where $x=(x_1,x_2)$ and $m,n$ are the size of the image. The purpose of introducing semi-local patches is to constrain the search space to a local area. The larger $\lambda$ leads to smaller search space making the searching quicker, while smaller $\lambda$ leads to global search and make more accurate results. Thus following \cite{shi2017weighted} we gradually reduce $\lambda$  by $\lambda^{k+1}=\max(\lambda^k-1,3)$ and initialization $\lambda=10$.

We apply our algorithm to  12 widely used test images. In our experiment, we select the patch size to be $11\times 11$. For each patch, the nearest neighbors are
obtained by using an approximate nearest neighbor (ANN) search algorithm.
We use a k-d tree approach as well as an ANN search algorithm to reduce
the computational cost. The linear system in weighted nonlocal Laplacian and
graph Laplacian is solved by the conjugate gradient method. We use the solution of WNLL after 6 steps as the initialization of our algorithm to get a proper initial guess of the similarity relationships between different groups. The initial image of WNLL is obtained by filling the missing pixels with random
numbers which satisfy a Gaussian distribution, where $\mu_0$ is the mean of $f|_{\Omega}$
and $\sigma_0$ is the standard deviation of $f|_{\Omega}$.\\

\begin{figure}
\centering
\includegraphics[scale=0.1]{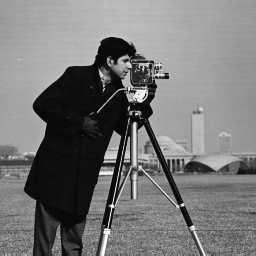}
\includegraphics[scale=0.1]{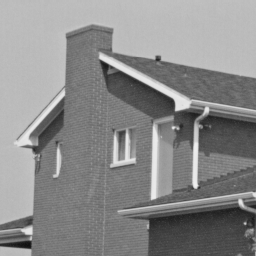}
\includegraphics[scale=0.1]{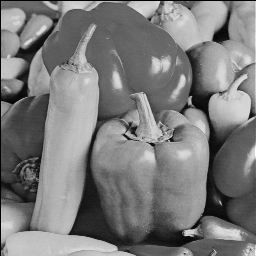}
\includegraphics[scale=0.1]{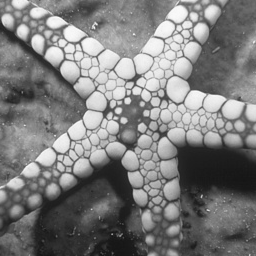}
\includegraphics[scale=0.1]{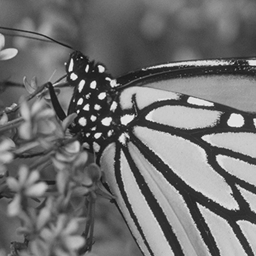}
\includegraphics[scale=0.1]{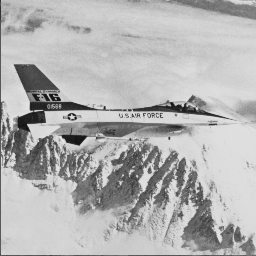}
\includegraphics[scale=0.1]{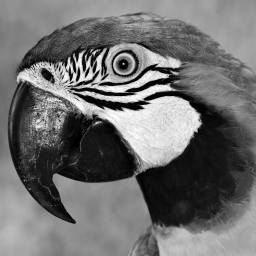}
\includegraphics[scale=0.05]{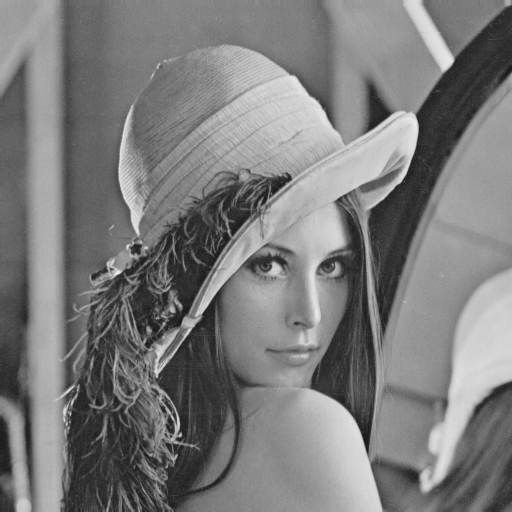}
\includegraphics[scale=0.05]{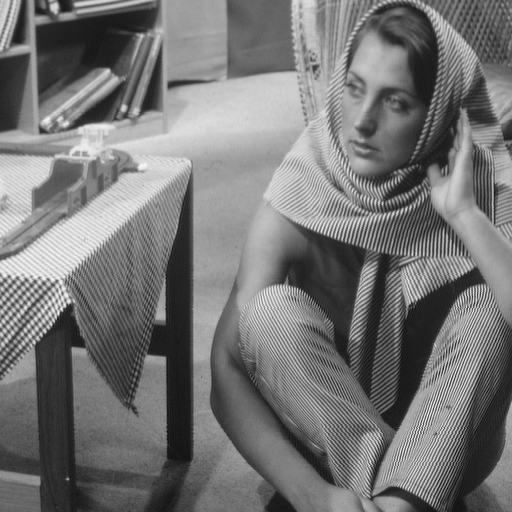}
\includegraphics[scale=0.8]{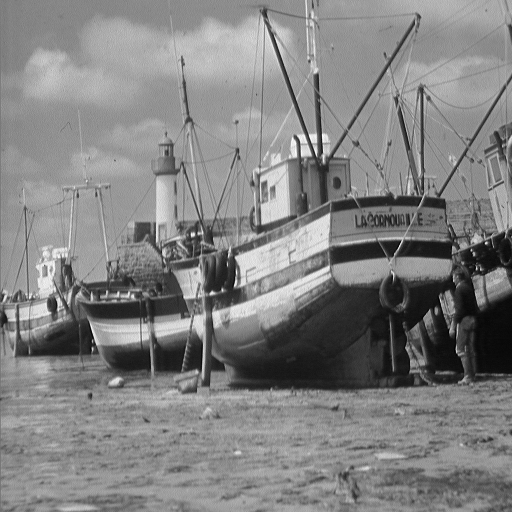}
\includegraphics[scale=0.05]{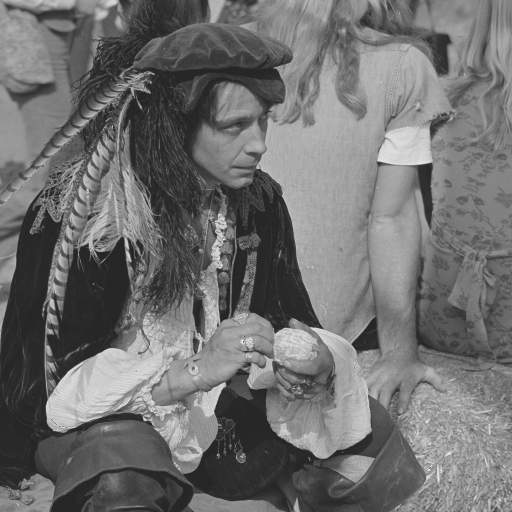}
\includegraphics[scale=0.05]{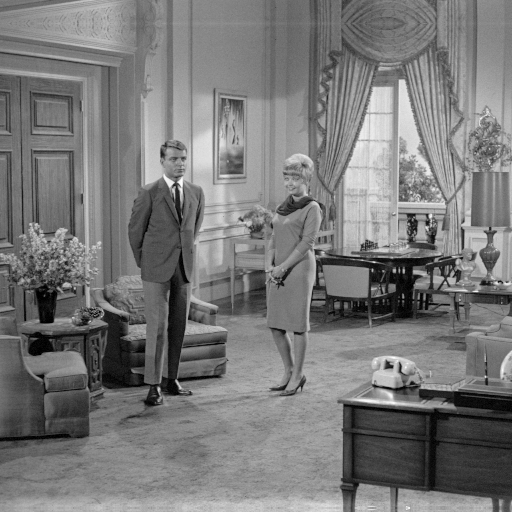}
\caption{\scriptsize{Set12: 12 widely used testing images.}}
\end{figure}

\indent Quality of the restored images is measured by PSNR and SSIM. PSNR is defined as
\begin{equation}
{\rm PSNR}(f,f^*) = -20\ {\rm log}_{10}(\left\|f-f^*\right\|/255)
\end{equation}
where $f^{*}$ is the ground truth. SSIM is defined as a multiplication of three terms that quantifies similarity of luminance, contrast and structure. It takes the following form
\begin{equation}
SSIM(x,y)=[l(x,y)]^\alpha\cdot [c(x,y)]^\beta\cdot [s(x,y)]^\gamma,
\end{equation}
where
\begin{equation}
l(x,y)=\frac{2\mu_x\mu_y+C_1}{\mu_x^2+\mu_y^2+C_1},c(x,y)=\frac{2\sigma_x\sigma_y+C_2}{\sigma_x^2+\sigma_y^2+C_2},s(x,y)=\frac{\sigma_{xy}+C_3}{\sigma_x\sigma_y+C_3},
\end{equation}
where $\mu_{x},\mu_{y},\sigma_{x},\sigma_{x}$ and $\sigma_{xy}$ are the local means, standard deviations and cross-covariance for image $x,y$.\\

\indent The numerical results are shown in Table \ref{table:1} and Table \ref{table:2}. For qualitative comparisons, Figure \ref{fig1} shows the inpainting results of 3 images from Set12 dataset at $15\%$ sample rate. Figure \ref{fig2} shows the inpainting results at $20\%$ sample rate. As we can see, WeCURE gives much better results than WNLL both visually and in terms of PSNR and SSIM. We observe that (We)CURE can well recover texture and preserve sharp image features such as edges, though it also introduces mild artifacts in smooth regions. This is why (We)CURE significantly outperforms WNLL in terms of SSIM.

\begin{table}[H]

\resizebox{\textwidth}{30mm}{
\begin{tabular}{@{}|c|cccccccccccc||c|@{}}
\toprule
Images      & C.man                          & House                          & Peppers                        & Starfish                       & Monarch                        & Airplane                       & Parrot                         & Lena                           & Barbara                        & Boat                           & Man                            & Couple                         & Average                        \\ \midrule
Sample Rate & \multicolumn{13}{c|}{10\%}                                                                                                                                                                                                                                                                                                                                                                                                                 \\ \midrule
LDMM          & 19.9329                        & 24.8723                        & 20.6103                        & 19.9285                        & 19.3395                        & 19.9612                        & 19.5449                        & 26.1005                        & 23.3176                        & 22.6681                        & 23.9415                        & 22.7225                        & 21.9117                        \\ \midrule
WNLL        & {\color[HTML]{FF0000} \textbf{21.9993}} & {\color[HTML]{0000FF} \underline{28.3325}} & 23.3210                        & {\color[HTML]{0000FF} \underline{22.2705}} & {\color[HTML]{0000FF} \underline{22.4218}} & {\color[HTML]{FF0000} \textbf{21.7954}} & {\color[HTML]{0000FF} \underline{21.6121}} & {\color[HTML]{0000FF} \underline{28.5089}} & {\color[HTML]{0000FF} \underline{26.3732}} & {\color[HTML]{0000FF} \underline{24.8116}} & {\color[HTML]{FF0000} \textbf{25.8126}} & {\color[HTML]{0000FF} \underline{25.0263}} & {\color[HTML]{0000FF} \underline{24.3571}} \\ \midrule
CURE         & 21.7095                        & 28.3023                        & {\color[HTML]{0000FF} \underline{23.3315}} & 22.0185                        & 22.0650                        & 21.4078                        & 21.5080                        & 28.3013                        & 26.3031                        & 24.6798                        & {\color[HTML]{0000FF} \underline{25.7207}} & 24.9033                        & 24.1876                        \\ \midrule
WeCURE        & {\color[HTML]{0000FF} \underline{21.8571}} & {\color[HTML]{FF0000}\textbf{ 28.7967}} & {\color[HTML]{FF0000} \textbf{23.7416}} & {\color[HTML]{FF0000} \textbf{22.3540}} & {\color[HTML]{FF0000} \textbf{22.5829}} & {\color[HTML]{0000FF} \underline{21.4335}} & {\color[HTML]{FF0000} \textbf{21.7753}} & {\color[HTML]{FF0000} \textbf{28.7926}} & {\color[HTML]{FF0000} \textbf{26.7155}} & {\color[HTML]{FF0000} \textbf{25.0060}} & 25.7145                        & {\color[HTML]{FF0000} \textbf{25.1940}} & {\color[HTML]{FF0000} \textbf{24.4970}} \\ \midrule
Sample Rate & \multicolumn{13}{c|}{15\%}                                                                                                                                                                                                                                                                                                                                                                                                                 \\ \midrule
LDMM         & 21.0948                        & 26.4075                        & 21.6434                        & 20.9887                        & 20.9843                        & 21.0712                        & 21.3412                        & 27.7591                        & 25.6175                        & 23.8791                        & 25.1269                        & 24.0065                        & 23.3267                        \\ \midrule
WNLL        & {\color[HTML]{FF0000} \textbf{23.3052}} & 29.1647                        & 25.0635                        & {\color[HTML]{0000FF} \underline{23.5147}} & 23.7171                        & {\color[HTML]{0000FF} \underline{22.7292}} & {\color[HTML]{0000FF} \underline{22.5851}} & 29.5856                        & {\color[HTML]{0000FF} \underline{27.7837}} & {\color[HTML]{0000FF} \underline{25.8633}} & {\color[HTML]{0000FF} \underline{26.9433}} & {\color[HTML]{0000FF} \underline{26.2245}} & {\color[HTML]{0000FF} \underline{25.5400}} \\ \midrule
CURE         & 22.8514                        & {\color[HTML]{0000FF} \underline{29.5745}} & {\color[HTML]{0000FF} \underline{25.1007}} & 23.4509                        & {\color[HTML]{0000FF} \underline{23.8326}} & 22.5211                        & 22.4579                        & {\color[HTML]{0000FF} \underline{29.6253}} & 27.7315                        & 25.7653                        & 26.9278                        & 26.1798                        & 25.5016                        \\ \midrule
WeCURE        & {\color[HTML]{0000FF} \underline{23.0993}} & {\color[HTML]{FF0000} \textbf{30.9540}} & {\color[HTML]{FF0000} \textbf{25.7840}} & {\color[HTML]{FF0000} \textbf{24.0722}} & {\color[HTML]{FF0000} \textbf{24.2587}} & {\color[HTML]{FF0000} \textbf{22.8246}} & {\color[HTML]{FF0000} \textbf{22.8708}} & {\color[HTML]{FF0000} \textbf{30.1331}} & {\color[HTML]{FF0000} \textbf{28.5615}} & {\color[HTML]{FF0000} \textbf{26.2943}} & {\color[HTML]{FF0000} \textbf{27.3484}} & {\color[HTML]{FF0000} \textbf{26.7266}} & {\color[HTML]{FF0000} \textbf{26.0773}} \\ \midrule
Sample Rate & \multicolumn{13}{c|}{20\%}                                                                                                                                                                                                                                                                                                                                                                                                                 \\ \midrule
LDMM          & 21.9057                        & 28.2924                        & 22.7767                        & 22.6264                        & 22.4175                        & 22.1073                        & 21.9409                        & 28.9160                        & 26.8121                        & 24.8777                        & 26.2350                        & 25.0044                        & 24.4927                        \\ \midrule
WNLL        & {\color[HTML]{0000FF} \underline{23.9478}} & 30.8222                        & {\color[HTML]{0000FF} \underline{25.8068}} & 24.5382                        & 24.6738                        & {\color[HTML]{0000FF} \underline{23.8359}} & 23.2844                        & 30.5140                        & 28.7357                        & 26.6614                        & 27.7806                        & 26.7532                        & 26.4462                        \\ \midrule
CURE         & 23.7846                        & {\color[HTML]{0000FF} \underline{31.4606}} & 25.7513                        & {\color[HTML]{0000FF} \underline{24.7232}} & {\color[HTML]{0000FF} \underline{24.8360}} & 23.7147                        & {\color[HTML]{0000FF} \underline{23.5282}} & {\color[HTML]{0000FF} \underline{30.6271}} & {\color[HTML]{0000FF} \underline{28.9715}} & {\color[HTML]{0000FF} \underline{26.6736}} & {\color[HTML]{0000FF} \underline{27.8198}} & {\color[HTML]{0000FF} \underline{26.8165}} & {\color[HTML]{0000FF} \underline{26.5589}} \\ \midrule
WeCURE        &{\color[HTML]{FF0000} \textbf{24.5007}} & {\color[HTML]{FF0000} \textbf{32.1789}} & {\color[HTML]{FF0000} \textbf{26.6428}} & {\color[HTML]{FF0000} \textbf{25.3982}} & {\color[HTML]{FF0000} \textbf{25.5151}} & {\color[HTML]{FF0000} \textbf{24.1406}} & {\color[HTML]{FF0000} \textbf{24.0625}} & {\color[HTML]{FF0000} \textbf{31.3711}} & {\color[HTML]{FF0000} \textbf{29.7794}} & {\color[HTML]{FF0000} \textbf{27.3033}} & {\color[HTML]{FF0000} \textbf{28.3473}} & {\color[HTML]{FF0000} \textbf{27.4934}} & {\color[HTML]{FF0000} \textbf{27.2278}} \\ \bottomrule
\end{tabular}
}
\caption{\scriptsize{The PSNR(dB) results of different methods on Set12 dataset with sampling rate $10\%$, $15\%$ and $20\%$. The best results are indicated in \textcolor{red}{red} and are highlighted in bold.
The second best results are indicated in \textcolor{blue}{blue} and are highlighted by underline.
}}
\label{table:1}
\end{table}

\begin{table}[H]
\resizebox{\textwidth}{30mm}{
\begin{tabular}{@{}|c|cccccccccccc||c|@{}}
\toprule
Images      & C.man                         & House                         & Peppers                       & Starfish                      & Monarch                       & Airplane                      & Parrot                        & Lena                          & Barbara                       & Boat                          & Man                           & Couple                        & Average                       \\ \midrule
Sample Rate & \multicolumn{13}{c|}{10\%}                                                                                                                                                                                                                                                                                                                                                                                                    \\ \midrule
LDMM          & 0.2677                        & 0.3406                        & 0.4406                        & 0.3856                        & 0.4870                        & 0.3338                        & 0.4560                        & 0.4508                        & 0.4881                        & 0.3121                        & 0.3469                        & 0.3389                        & 0.3874                        \\ \midrule
WNLL        & {\color[HTML]{000000} 0.3557} & {\color[HTML]{000000} 0.4236} & {\color[HTML]{000000} 0.5681} & {\color[HTML]{0000FF} \underline{0.5415}} & {\color[HTML]{000000} 0.6523} & {\color[HTML]{0000FF} \underline{0.4352}} & {\color[HTML]{000000} 0.5680} & {\color[HTML]{000000} 0.5316} & {\color[HTML]{000000} 0.6308} & {\color[HTML]{000000} 0.4383} & {\color[HTML]{000000} 0.4787} & {\color[HTML]{000000} 0.5123} & {\color[HTML]{000000} 0.5113} \\ \midrule
CURE         & {\color[HTML]{0000FF} \underline{0.3591}} & {\color[HTML]{0000FF} \underline{0.4337}} & {\color[HTML]{0000FF} \underline{0.5849}} & {\color[HTML]{000000} 0.5382} & {\color[HTML]{0000FF} \underline{0.6537}} & {\color[HTML]{000000} 0.4324} & {\color[HTML]{0000FF} \underline{0.5733}} & {\color[HTML]{0000FF} \underline{0.5356}} & {\color[HTML]{0000FF} \underline{0.6392}} & {\color[HTML]{0000FF} \underline{0.4409}} & {\color[HTML]{0000FF} \underline{0.4817}} & {\color[HTML]{0000FF} \underline{0.5240}} & {\color[HTML]{0000FF} \underline{0.5164}} \\ \midrule
WeCURE        & {\color[HTML]{FF0000} \textbf{0.3726}} & {\color[HTML]{FF0000} \textbf{0.4397}} & {\color[HTML]{FF0000} \textbf{0.6042}} & {\color[HTML]{FF0000} \textbf{0.5721}} & {\color[HTML]{FF0000} \textbf{0.6842}} & {\color[HTML]{FF0000} \textbf{0.4448}} & {\color[HTML]{FF0000} \textbf{0.5953}} & {\color[HTML]{FF0000} \textbf{0.5402}} & {\color[HTML]{FF0000} \textbf{0.6572}} & {\color[HTML]{FF0000} \textbf{0.4628}} & {\color[HTML]{FF0000} \textbf{0.5051}} & {\color[HTML]{FF0000} \textbf{0.5476}} & {\color[HTML]{FF0000} \textbf{0.5355}} \\ \midrule
Sample Rate & \multicolumn{13}{c|}{15\%}                                                                                                                                                                                                                                                                                                                                                                                                    \\ \midrule
LDMM          & 0.3622                        & 0.4288                        & 0.5308                        & 0.4848                        & 0.5986                        & 0.4252                        & 0.5464                        & 0.5382                        & 0.6164                        & 0.4187                        & 0.4483                        & 0.4619                        & 0.4884                        \\ \midrule
WNLL        & {\color[HTML]{000000} 0.4456} & {\color[HTML]{000000} 0.5053} & {\color[HTML]{000000} 0.6380} & {\color[HTML]{000000} 0.6196} & {\color[HTML]{000000} 0.7076} & {\color[HTML]{000000} 0.5052} & {\color[HTML]{000000} 0.6247} & {\color[HTML]{000000} 0.5931} & {\color[HTML]{000000} 0.6964} & {\color[HTML]{000000} 0.5130} & {\color[HTML]{000000} 0.5544} & {\color[HTML]{000000} 0.5911} & {\color[HTML]{000000} 0.5828} \\ \midrule
CURE        & {\color[HTML]{0000FF} \underline{0.4464}} & {\color[HTML]{0000FF} \underline{0.5294}} & {\color[HTML]{0000FF} \underline{0.6610}} & {\color[HTML]{0000FF} \underline{0.6294}} & {\color[HTML]{0000FF} \underline{0.7299}} & {\color[HTML]{0000FF} \underline{0.5115}} & {\color[HTML]{0000FF} \underline{0.6435}} & {\color[HTML]{0000FF} \underline{0.5994}} & {\color[HTML]{0000FF} \underline{0.7068}} & {\color[HTML]{0000FF} \underline{0.5226}} & {\color[HTML]{0000FF} \underline{0.5637}} & {\color[HTML]{0000FF} \underline{0.6067}} & {\color[HTML]{0000FF} \underline{0.5959}} \\ \midrule
WeCURE        & {\color[HTML]{FF0000} \textbf{0.4577}} & {\color[HTML]{FF0000} \textbf{0.5459}} & {\color[HTML]{FF0000} \textbf{0.6766}} & {\color[HTML]{FF0000} \textbf{0.6658}} & {\color[HTML]{FF0000} \textbf{0.7473}} & {\color[HTML]{FF0000} \textbf{0.5273}} & {\color[HTML]{FF0000} \textbf{0.6621}} & {\color[HTML]{FF0000} \textbf{0.6102}} & {\color[HTML]{FF0000} \textbf{0.7275}} & {\color[HTML]{FF0000} \textbf{0.5462}} & {\color[HTML]{FF0000} \textbf{0.5939}} & {\color[HTML]{FF0000} \textbf{0.6308}} & {\color[HTML]{FF0000} \textbf{0.6159}} \\ \midrule
Sample Rate & \multicolumn{13}{c|}{20\%}                                                                                                                                                                                                                                                                                                                                                                                                    \\ \midrule
LDMM         & {\color[HTML]{000000} 0.4385} & {\color[HTML]{000000} 0.5148} & {\color[HTML]{000000} 0.5980} & {\color[HTML]{000000} 0.5783} & {\color[HTML]{000000} 0.6692} & {\color[HTML]{000000} 0.5003} & {\color[HTML]{000000} 0.6074} & {\color[HTML]{000000} 0.5997} & {\color[HTML]{000000} 0.6840} & {\color[HTML]{000000} 0.5003} & {\color[HTML]{000000} 0.5295} & {\color[HTML]{000000} 0.5501} & 0.5642                        \\ \midrule
WNLL        & {\color[HTML]{000000} 0.4970} & {\color[HTML]{000000} 0.5735} & {\color[HTML]{000000} 0.6856} & {\color[HTML]{000000} 0.6691} & {\color[HTML]{000000} 0.7439} & {\color[HTML]{000000} 0.5684} & {\color[HTML]{000000} 0.6673} & {\color[HTML]{000000} 0.6376} & {\color[HTML]{000000} 0.7373} & {\color[HTML]{000000} 0.5722} & {\color[HTML]{000000} 0.6062} & {\color[HTML]{000000} 0.6364} & 0.6329                        \\ \midrule
CURE         & {\color[HTML]{0000FF} \underline{0.5063}} & {\color[HTML]{0000FF} \underline{0.6044}} & {\color[HTML]{0000FF} \underline{0.7051}} & {\color[HTML]{0000FF} \underline{0.6889}} & {\color[HTML]{0000FF} \underline{0.7687}} & {\color[HTML]{0000FF} \underline{0.5847}} & {\color[HTML]{0000FF} \underline{0.6850}} & {\color[HTML]{0000FF} \underline{0.6457}} & {\color[HTML]{0000FF} \underline{0.7515}} & {\color[HTML]{0000FF} \underline{0.5882}} & {\color[HTML]{0000FF} \underline{0.6203}} & {\color[HTML]{0000FF} \underline{0.6571}} & {\color[HTML]{0000FF} \underline{0.6505}} \\ \midrule
WeCURE        & {\color[HTML]{FF0000} \textbf{0.5270}} & {\color[HTML]{FF0000} \textbf{0.6167}} & {\color[HTML]{FF0000} \textbf{0.7241}} & {\color[HTML]{FF0000} \textbf{0.7214}} & {\color[HTML]{FF0000} \textbf{0.7859}} & {\color[HTML]{FF0000} \textbf{0.6009}} & {\color[HTML]{FF0000} \textbf{0.7017}} & {\color[HTML]{FF0000} \textbf{0.6570}} & {\color[HTML]{FF0000} \textbf{0.7683}} & {\color[HTML]{FF0000} \textbf{0.6093}} & {\color[HTML]{FF0000} \textbf{0.6492}} & {\color[HTML]{FF0000} \textbf{0.6806}} & {\color[HTML]{FF0000} \textbf{0.6702}} \\ \bottomrule
\end{tabular}}
\caption{\scriptsize{The SSIM results of different methods on Set12 dataset with sampling rate $10\%$, $15\%$ and $20\%$. The best results are indicated in \textcolor{red}{red} and are highlighted in bold.
The second best results are indicated in \textcolor{blue}{blue} and are highlighted by underline.}}
\label{table:2}
\end{table}

\newpage

\begin{figure}[H]
\centering
\tiny
\subfigure{
\begin{minipage}[t]{0.25\linewidth}
\centering
\includegraphics[width=1.3in]{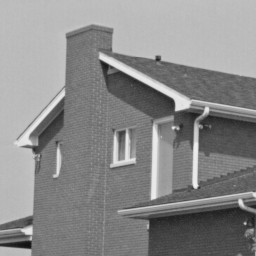}
\end{minipage}
\begin{minipage}[t]{0.25\linewidth}
\centering
\includegraphics[width=1.3in]{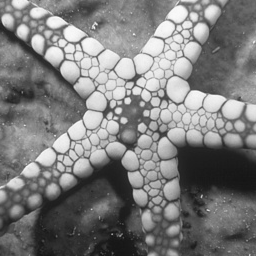}
\end{minipage}
\begin{minipage}[t]{0.25\linewidth}
\centering
\includegraphics[width=1.3in]{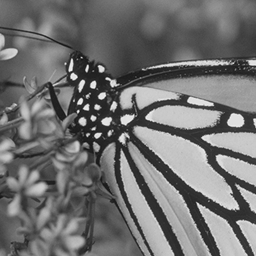}
\end{minipage}
}
\subfigure{
\begin{minipage}[t]{0.25\linewidth}
\centering
\includegraphics[width=1.3in]{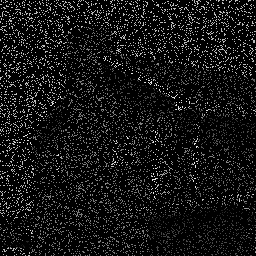}
\end{minipage}
\begin{minipage}[t]{0.25\linewidth}
\centering
\includegraphics[width=1.3in]{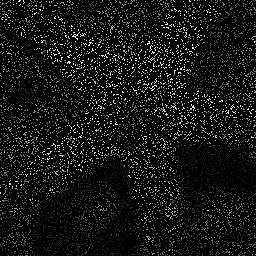}
\end{minipage}
\begin{minipage}[t]{0.25\linewidth}
\centering
\includegraphics[width=1.3in]{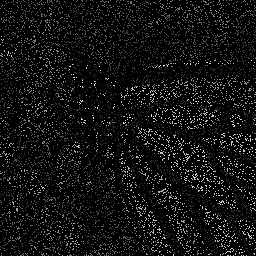}
\end{minipage}
}
\subfigure{
\begin{minipage}[t]{0.25\linewidth}
\centering
\includegraphics[width=1.3in]{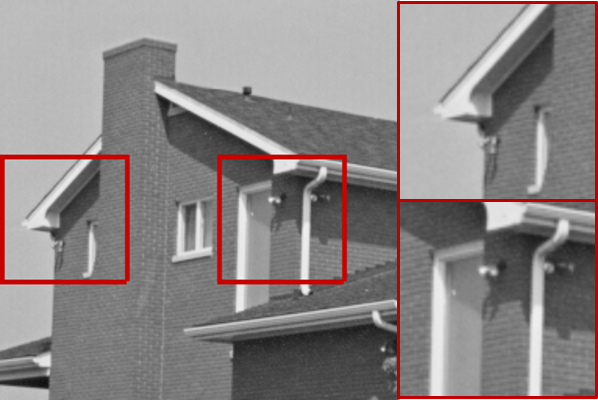}
\end{minipage}
\begin{minipage}[t]{0.25\linewidth}
\centering
\includegraphics[width=1.3in]{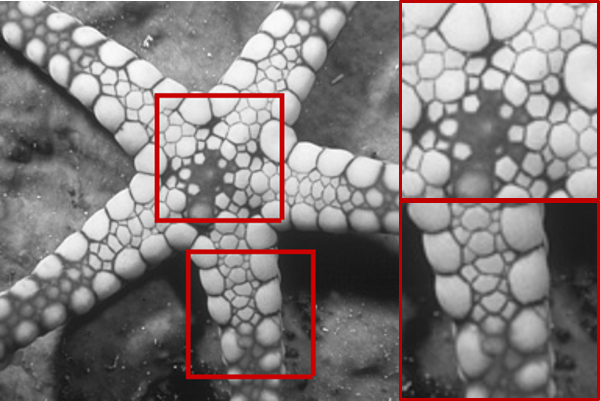}
\end{minipage}
\begin{minipage}[t]{0.25\linewidth}
\centering
\includegraphics[width=1.3in]{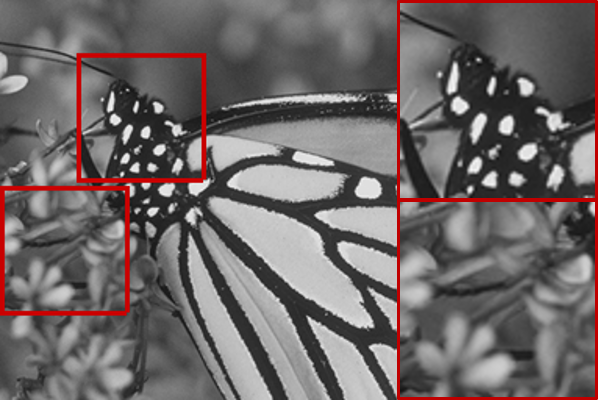}
\end{minipage}
}
\subfigure{
\begin{minipage}[t]{0.25\linewidth}
\centering
\includegraphics[width=1.3in]{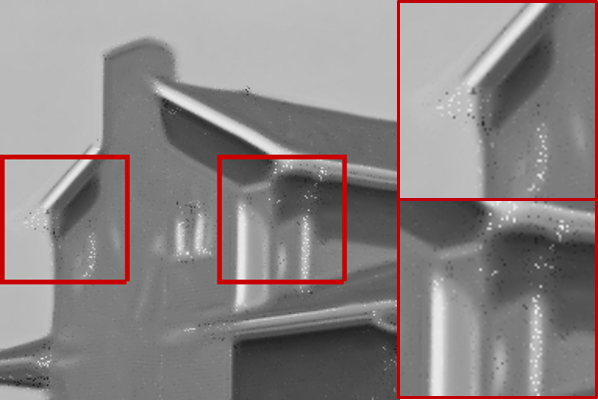}
\caption*{PSNR=26.41dB,SSIM=0.42}
\end{minipage}
\begin{minipage}[t]{0.25\linewidth}
\centering
\includegraphics[width=1.3in]{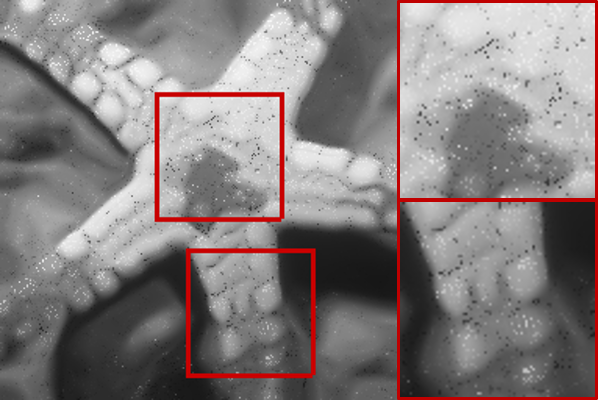}
\caption*{PSNR=20.99dB,SSIM=0.48}
\end{minipage}
\begin{minipage}[t]{0.25\linewidth}
\centering
\includegraphics[width=1.3in]{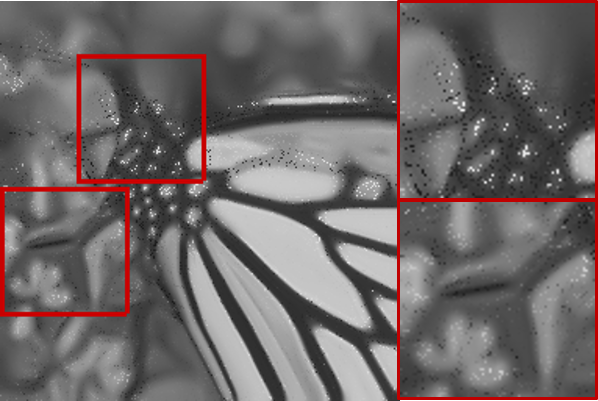}
\caption*{PSNR=20.98dB,SSIM=0.60}
\end{minipage}
}
\subfigure{
\begin{minipage}[t]{0.25\linewidth}
\centering
\includegraphics[width=1.3in]{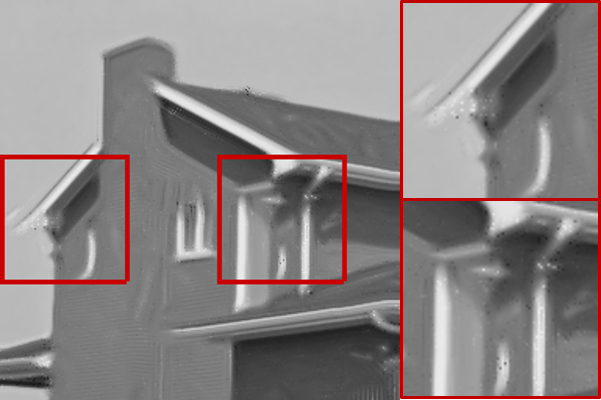}
\caption*{PSNR=29.16dB,SSIM=0.50}
\end{minipage}
\begin{minipage}[t]{0.25\linewidth}
\centering
\includegraphics[width=1.3in]{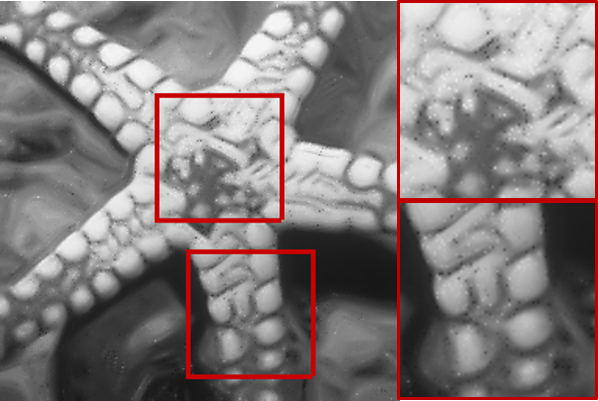}
\caption*{PSNR=23.51dB,SSIM=0.62}
\end{minipage}
\begin{minipage}[t]{0.25\linewidth}
\centering
\includegraphics[width=1.3in]{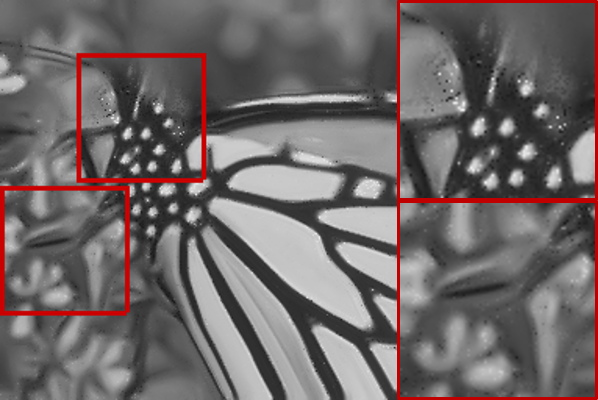}
\caption*{PSNR=23.71dB,SSIM=0.71}
\end{minipage}
}
\subfigure{
\begin{minipage}[t]{0.25\linewidth}
\centering
\includegraphics[width=1.3in]{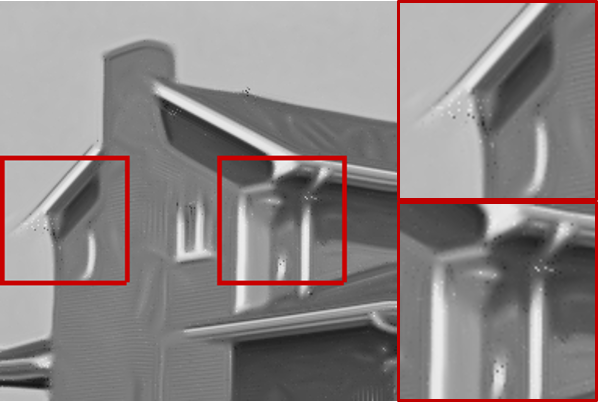}
\caption*{PSNR=29.57dB,SSIM=0.53}
\end{minipage}
\begin{minipage}[t]{0.25\linewidth}
\centering
\includegraphics[width=1.3in]{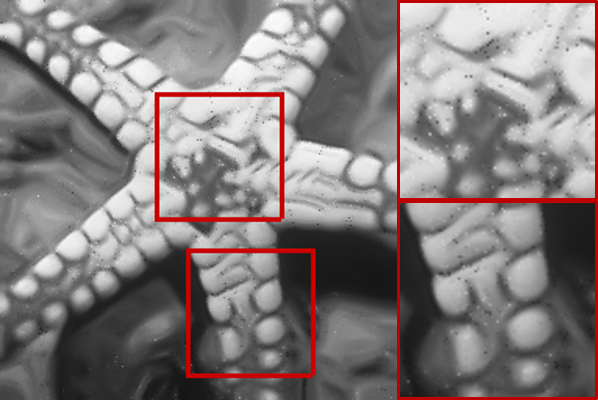}
\caption*{PSNR=23.45dB,SSIM=0.63}
\end{minipage}
\begin{minipage}[t]{0.25\linewidth}
\centering
\includegraphics[width=1.3in]{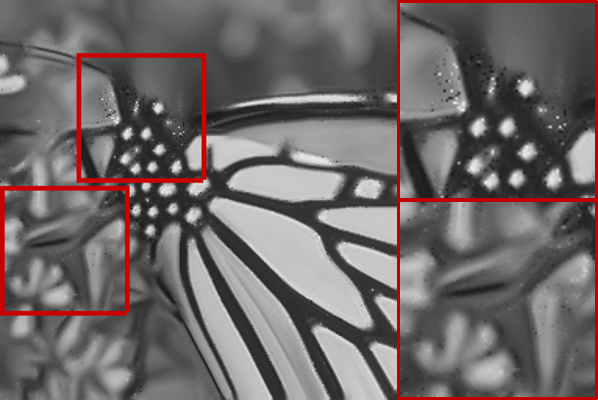}
\caption*{PSNR=23.83dB,SSIM=0.73}
\end{minipage}
}
\subfigure{
\begin{minipage}[t]{0.25\linewidth}
\centering
\includegraphics[width=1.3in]{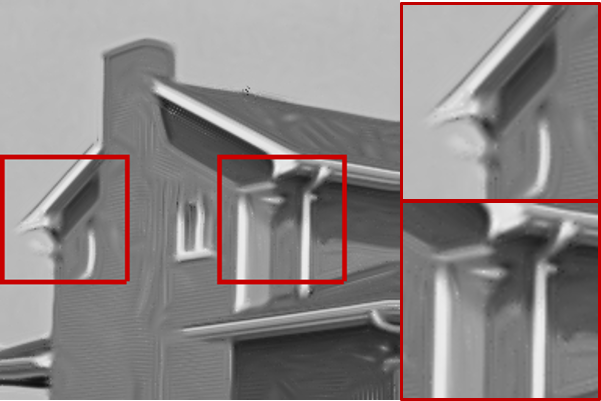}
\caption*{PSNR=30.95dB,SSIM=0.55}
\end{minipage}
\begin{minipage}[t]{0.25\linewidth}
\centering
\includegraphics[width=1.3in]{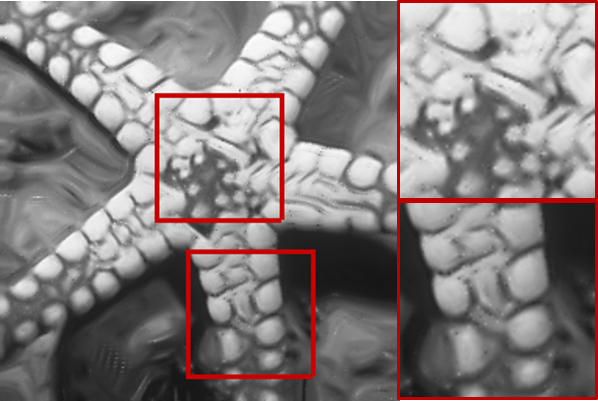}
\caption*{PSNR=24.07dB,SSIM=0.68}
\end{minipage}
\begin{minipage}[t]{0.25\linewidth}
\centering
\includegraphics[width=1.3in]{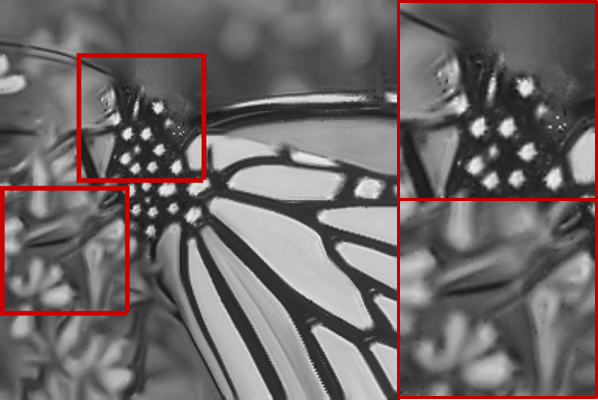}
\caption*{PSNR=24.26dB,SSIM=0.75}
\end{minipage}
}
\centering
\caption{From top to bottom: original image, $15\%$ subsample, ground-truth, LDMM, WNLL, CURE, WeCURE}
\label{fig1}
\end{figure}

\begin{figure}[H]
\centering

\subfigure{
\begin{minipage}[t]{0.25\linewidth}
\centering
\includegraphics[width=1.3in]{02.png}
\end{minipage}
\begin{minipage}[t]{0.25\linewidth}
\centering
\includegraphics[width=1.3in]{04.png}
\end{minipage}
\begin{minipage}[t]{0.25\linewidth}
\centering
\includegraphics[width=1.3in]{05.png}
\end{minipage}
}
\subfigure{
\begin{minipage}[t]{0.25\linewidth}
\centering
\includegraphics[width=1.3in]{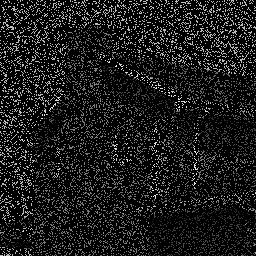}
\end{minipage}
\begin{minipage}[t]{0.25\linewidth}
\centering
\includegraphics[width=1.3in]{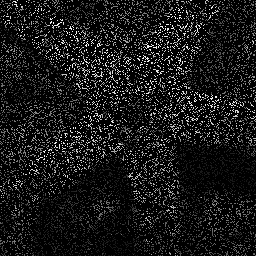}
\end{minipage}
\begin{minipage}[t]{0.25\linewidth}
\centering
\includegraphics[width=1.3in]{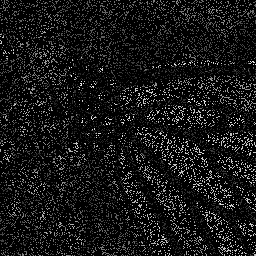}
\end{minipage}
}
\subfigure{
\begin{minipage}[t]{0.25\linewidth}
\centering
\includegraphics[width=1.3in]{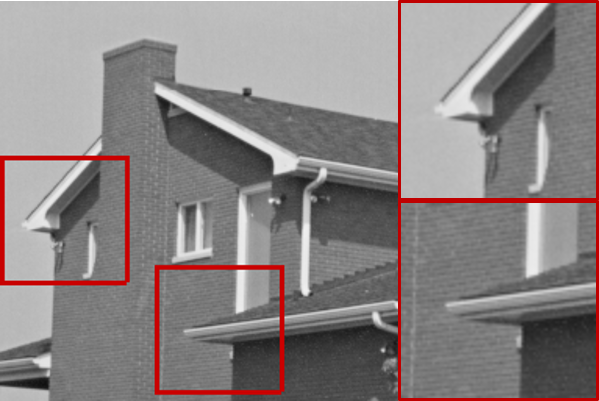}
\end{minipage}
\begin{minipage}[t]{0.25\linewidth}
\centering
\includegraphics[width=1.3in]{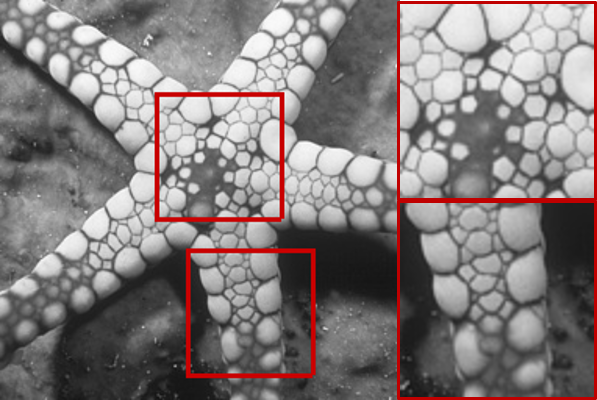}
\end{minipage}
\begin{minipage}[t]{0.25\linewidth}
\centering
\includegraphics[width=1.3in]{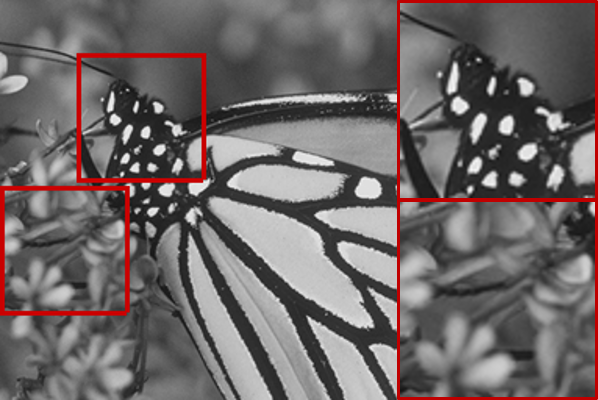}
\end{minipage}
}
\subfigure{
\begin{minipage}[t]{0.25\linewidth}
\centering
\includegraphics[width=1.3in]{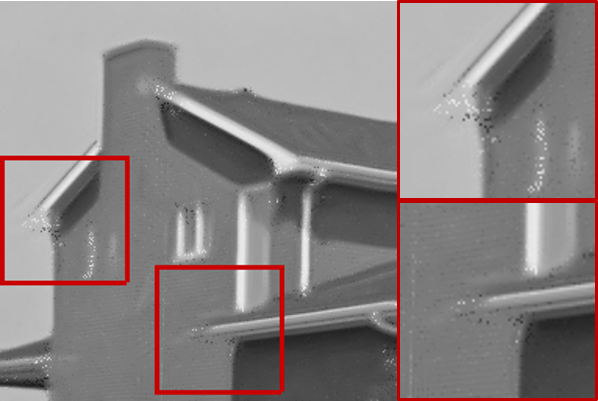}
\caption*{PSNR=28.29dB,SSIM=0.51}
\end{minipage}
\begin{minipage}[t]{0.25\linewidth}
\centering
\includegraphics[width=1.3in]{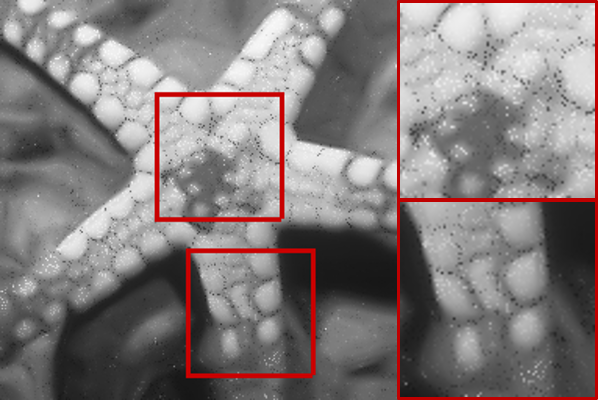}
\caption*{PSNR=22.63dB,SSIM=0.58}
\end{minipage}
\begin{minipage}[t]{0.25\linewidth}
\centering
\includegraphics[width=1.3in]{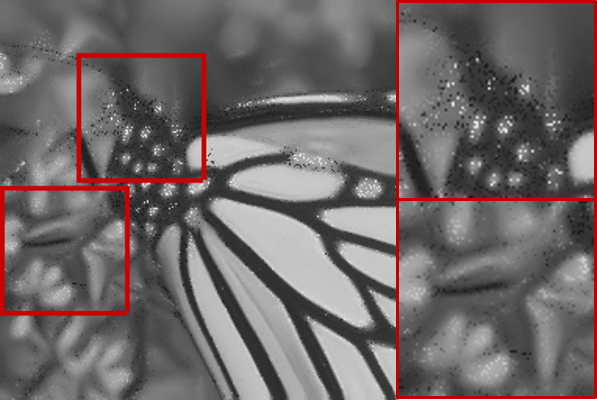}
\caption*{PSNR=22.42dB,SSIM=0.67}
\end{minipage}
}
\subfigure{
\begin{minipage}[t]{0.25\linewidth}
\centering
\includegraphics[width=1.3in]{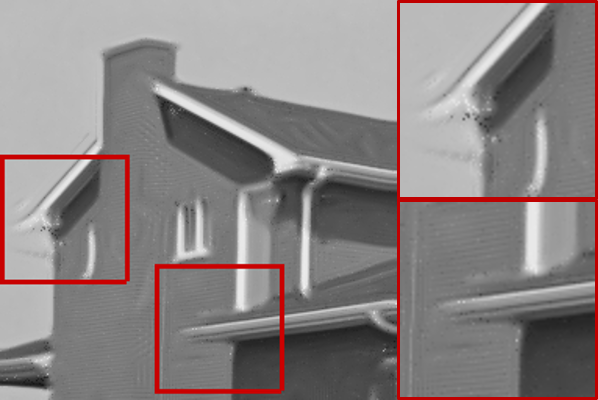}
\caption*{PSNR=30.82dB,SSIM=0.57}
\end{minipage}
\begin{minipage}[t]{0.25\linewidth}
\centering
\includegraphics[width=1.3in]{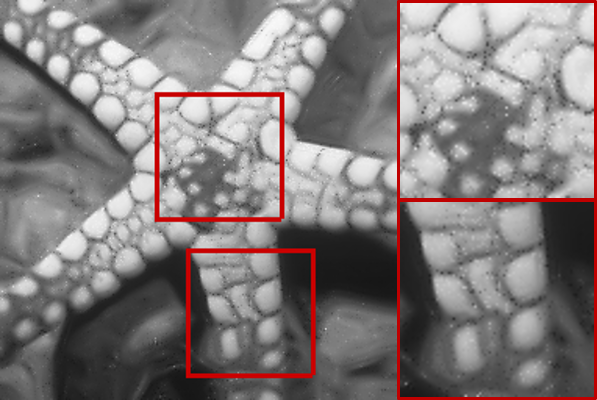}
\caption*{PSNR=24.54dB,SSIM=0.67}
\end{minipage}
\begin{minipage}[t]{0.25\linewidth}
\centering
\includegraphics[width=1.3in]{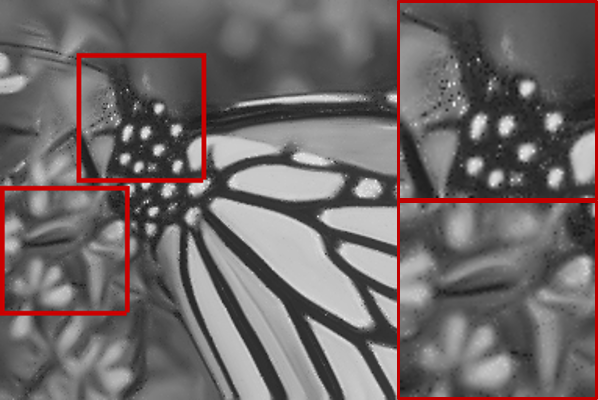}
\caption*{PSNR=24.67dB,SSIM=0.74}
\end{minipage}
}
\subfigure{
\begin{minipage}[t]{0.25\linewidth}
\centering
\includegraphics[width=1.3in]{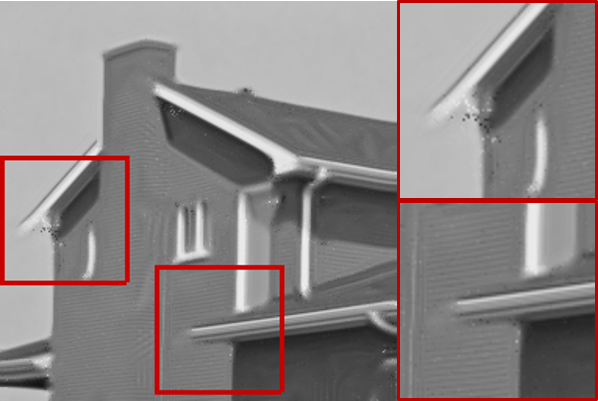}
\caption*{PSNR=31.46dB,SSIM=0.60}
\end{minipage}
\begin{minipage}[t]{0.25\linewidth}
\centering
\includegraphics[width=1.3in]{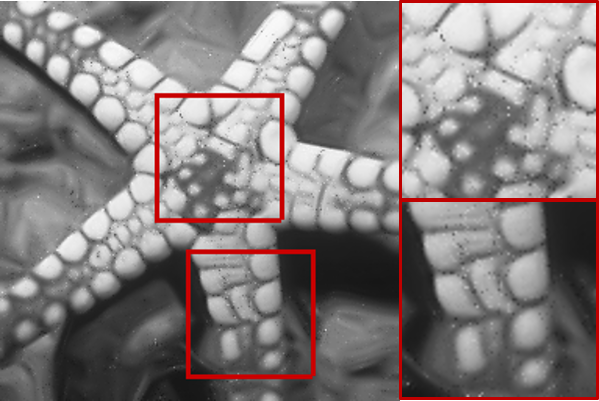}
\caption*{PSNR=24.72dB,SSIM=0.68}
\end{minipage}
\begin{minipage}[t]{0.25\linewidth}
\centering
\includegraphics[width=1.3in]{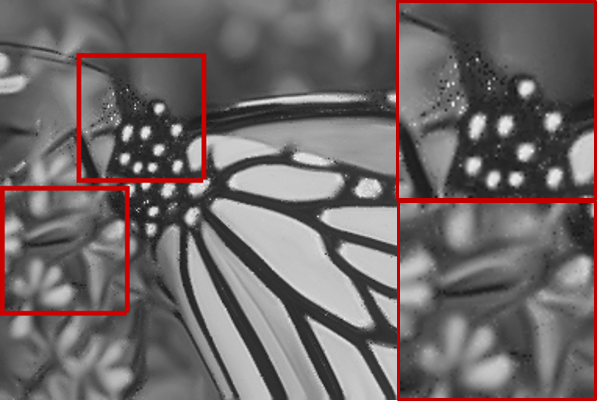}
\caption*{PSNR=24.83dB,SSIM=0.77}
\end{minipage}
}
\subfigure{
\begin{minipage}[t]{0.25\linewidth}
\centering
\includegraphics[width=1.3in]{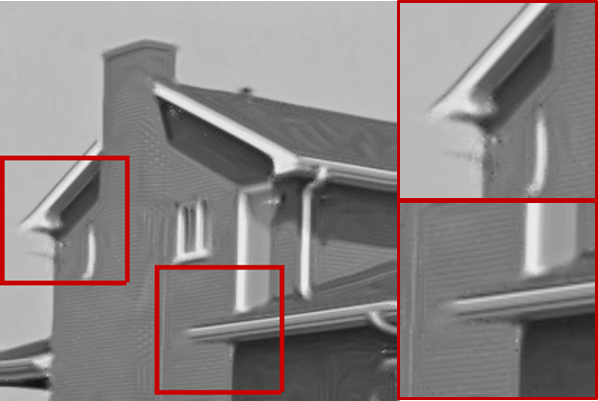}
\caption*{PSNR=32.18dB,SSIM=0.62}
\end{minipage}
\begin{minipage}[t]{0.25\linewidth}
\centering
\includegraphics[width=1.3in]{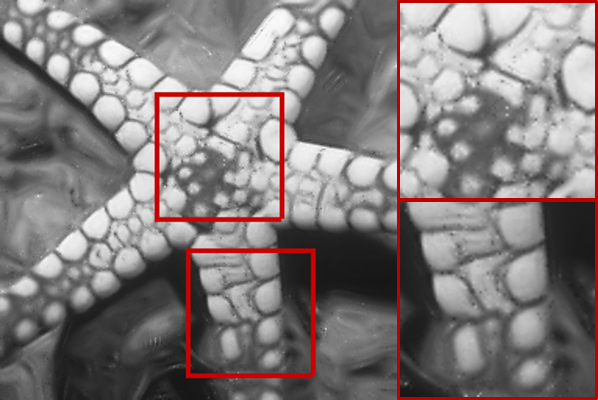}
\caption*{PSNR=25.39dB,SSIM=0.72}
\end{minipage}
\begin{minipage}[t]{0.25\linewidth}
\centering
\includegraphics[width=1.3in]{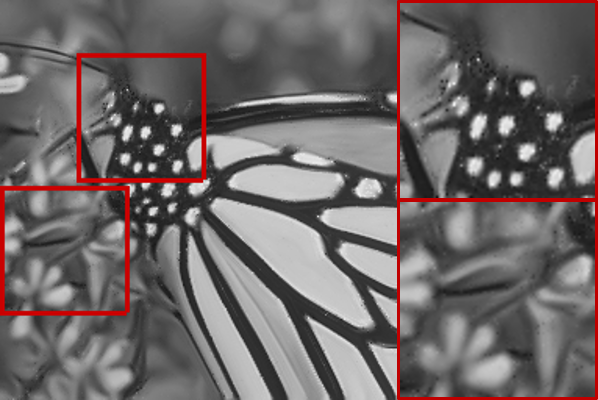}
\caption*{PSNR=25.51dB,SSIM=0.79}
\end{minipage}
}
\centering
\caption{From top to bottom: original image, $20\%$ subsample, ground-truth, LDMM, WNLL, CURE, WeCURE}
\label{fig2}
\end{figure}

\newpage

\section{Asymptotic Analysis}
\label{sec:ana}
In this section, we aim to provide an asymptotic analysis of the proposed numerical scheme for WeCURE model using $\Gamma$-convergence. The idea of the proof is sketched as follows. We first fix the bandwidth of the kernel and consider our scheme as an integral scheme of a non-local functional. Then, we reduce the bandwidth of the kernel to zero to show that the non-local functional is a good approximation to the original WeCURE functional. The proof mostly follows the notation and general idea of \cite{trillos2016continuum,trillos2018variational,trillos2018error}. A recent paper \cite{dunlop2019large} also established a $\Gamma$-convergence proof of the Biharmonic equation. The difference between their paper and ours is mainly the setting of the problem. In their paper, labeled data is considered as the boundary condition, while in our paper, we also consider the labeled data as samples from the data distribution and the rate of the number of labeled and unlabeled data is a fixed factor. In this setting, we will show that weights of WeCURE are crucial to achieving convergence.

Let $P=\{x_1,x_2,\cdots,x_n\}$ and $x_{i} (1\leq i\leq n)$ be uniformly sampled from $\Omega$, where $\Omega$ is an open bounded domain in $\mathbb{R}^{d}$. Let $\{x_{i_1},x_{i_2},\cdots,x_{i_m}\}$ be the set of labeled points where $x_{i_j}(1\leq j\leq m)$ is uniformly sampled from $P$. In this paper, we consider the ratio $\gamma =\frac{n}{m}$ to be fixed. Let $b:\Omega\rightarrow\mathbb{R}$ be a function whose value is only known at the labeled points. The empirical measure of data points is given by $\mu_{n}=\frac{1}{n}\sum_{i=1}^{n}\delta_{x_i}$. We consider a graph with vertices $V=P$ and denote the weights of the edges as $W_{ij}=\eta_\varepsilon(x_i-x_j)$ where $\eta_\varepsilon(x):=\eta_\varepsilon(|x|)=\frac{1}{\varepsilon^d}\eta(\frac{|x|}{\varepsilon}),\ \eta: [0,+\infty)\rightarrow[0,+\infty)$ is a radially symmetric function which satisfies the following assumptions:\\
\textbf{(A1)} $\eta(0)>0$ and $\eta$ is continuous at 0.\\
\textbf{(A2)} $\eta$ is non-increasing.\\
\textbf{(A3)} $\eta$ has compact support. If $|r|>\alpha$, then $\eta(r)=0$.

The discrete WeCURE model is given by (the weight is $\gamma -1$, not $\gamma$ in previous sections)
\begin{subequations}
\begin{align*}
&WeCURE_{n,\varepsilon}(u)\\
&=\frac{1}{\varepsilon^2}\frac{1}{(n-m)^2}\sum_{i, j \in P\backslash S}W_{ij}(u(x_i)-u(x_j))^2\\
&+\frac{2}{(n-m)m}\sum_{i\in S, j\in P\backslash S}W_{ij}(u(x_i)-b(x_j))^2+\frac{\lambda}{\varepsilon^4(n-m)^3}\sum_{i\in P\backslash S}\left(\sum_{j\in P\backslash S}W_{ij}(u(x_i)-u(x_j))\right)^2\label{eq:disb}\\
&+\frac{2}{\varepsilon^2(n-m)^2m(\gamma-1)}\sum_{i\in P\backslash S}\left(\sum_{j\in P\backslash S}W_{ij}(u(x_i)-u(x_j))\sum_{j\in S}W_{i,j}(u(x_i)-b(x_j))\right)\\
&+\frac{1}{(n-m)m^2(\gamma-1)^2}\sum_{i\in P\backslash S}\left(\sum_{j\in S}W_{ij}(u(x_i)-b(x_j))\right)^2+\frac{1}{(n-m)^2m}\sum_{i\in S}\left(\sum_{j\in P\backslash S}W_{ij}(b(x_i)-u(x_j))\right)^2\\
&+\frac{2}{\varepsilon^2(n-m)m^2(\gamma-1)}\sum_{i\in S}\left(\sum_{j\in P\backslash S}W_{ij}(b(x_i)-u(x_j))\sum_{j\in S}W_{ij}(b(x_i)-b(x_j))\right)
\end{align*}
\end{subequations}
The continuum nonlocal WeCURE model is given by
\begin{subequations}
\begin{align}
WeCURE_{\varepsilon}(u)&=\frac{1}{\varepsilon^2}\int_{\Omega}\int_{\Omega}\eta_\varepsilon(x-y)(u(x)-u(y))^2dydx\label{n1}\\
&+2\int_{\Omega}\int_{\Omega}\eta_\varepsilon(x-y)(u(x)-b(y))^2dydx\label{n2} \\
&+\lambda(\frac{1}{\varepsilon^4}\int_{\Omega}(\int_{\Omega}\eta_\varepsilon(x-y)(u(x)-u(y))dy)^2dx\label{n3}\\
&+\frac{2}{\varepsilon^2}\frac{1}{\gamma-1}\int_{\Omega}(\int_{\Omega}\eta_\varepsilon(x-y)(u(x)-u(y))dy\int_{\Omega}\eta_\varepsilon(x-y)(u(x)-b(y))dy)dx\label{n4}\\
&+\frac{1}{(\gamma-1)^2}\int_{\Omega}(\int_{\Omega}\eta_\varepsilon(x-y)(u(x)-b(y))dy)^2dx\label{n5}\\
&+\int_{\Omega}(\int_{\Omega}\eta_\varepsilon(x-y)(b(x)-u(y))dy)^2dx\label{n6}\\
&+\frac{2}{\varepsilon^2}\frac{1}{\gamma-1}\int_{\Omega}(\int_{\Omega}\eta_\varepsilon(x-y)(b(x)-u(y))dy\int_{\Omega}\eta_\varepsilon(x-y)(b(x)-b(y))dy)dx)\label{n7}
\end{align}\label{eq2}
\end{subequations}
The continuum (local) WeCURE model is given by
\begin{subequations}
\begin{align}
WeCURE(u)&=2\sigma_\eta\int_{\Omega}|\nabla u(x)|^2dx+2\int_{\Omega}(u(x)-b(x))^2dx\\
&+\lambda(\sigma_\eta^2\int_{\Omega}(\Delta u(x))^2dx\\
&+\frac{2\sigma_\eta}{\gamma-1}\int_{\Omega}(-\Delta u(x))(u(x)-b(x))dx\\
&+(\frac{1}{(\gamma-1)^2}+1)\int_{\Omega}(u(x)-b(x))^2dx\\
&+\frac{2\sigma_\eta}{\gamma-1}\int_{\Omega}(b(x)-u(x))(-\Delta b(x))dx)\\
&=\lambda\sigma_\eta^2\int_{\Omega}(\Delta u(x))^2dx\label{l1}\\
&+2\sigma_\eta\int_{\Omega}|\nabla u(x)|^2dx\label{l2}\\
&+\frac{2\lambda\sigma_\eta}{\gamma-1}\int_\Omega|\nabla(u(x)-b(x))|^2dx\label{l3}\\
&+(2+\lambda+\frac{\lambda}{(\gamma-1)^2})\int_{\Omega}(u(x)-b(x))^2dx\label{l4}
\end{align}\label{eq3}
\end{subequations}
where $\sigma_\eta=\frac{1}{2}\int_{\mathbb{R}^d}\eta(h)|h_1|^2dh$, $h_1$ is the first coordinate of vector $h$.
\begin{remark}
The models introduced above contain the corresponding CURE models as special cases if we simply modify some coefficients in the WeCURE models and replace the term $\int_\Omega|\nabla(u(x)-b(x))|^2dx$ by $\int_\Omega|\nabla(u(x)-c\cdot b(x))|^2dx$ ($c\ne1$ is a certain constant).\label{rmk1}
\end{remark}

We are now ready to present the main theorems of this section.
\begin{theorem}
Let $\Omega\subset \mathbb{R}^d$, $d\geq2$ be an open, bounded, connected set with Lipschitz boundary. Let $x_1,\cdots,x_n,\cdots$ be a sequence of i.i.d random points uniformly sampled from $\Omega$. $S=\{x_{i_1},x_{i_2},\cdots,x_{i_m}: x_{i_j} \text{uniformly sampled from } x_1,\cdots,x_n,\cdots\}$ is the set of labeled points whose value is given by $b(x_{i_j})$. Assume the kernel $\eta$ satisfies conditions (A1)-(A3). Then $WeCURE_{n,\varepsilon}$ $\Gamma$-converges to $WeCURE_{\varepsilon}$ as $n\rightarrow\infty$ in the $TL^2$ sense.\label{thm1}
\end{theorem}

\begin{theorem}
Under the assumptions of \cref{thm1}, $WeCURE_{\varepsilon}$ $\Gamma$-converges to $WeCURE$ as $\varepsilon\rightarrow0$ in $H_0^2(\Omega)$ with $L^2(\Omega)$ metric.\label{thm2}
\end{theorem}

\begin{theorem}
(Compactness)Under the assumptions of \cref{thm1}, $\{WeCURE_{\varepsilon}\}_{\varepsilon>0}$ satisfies the compactness property with respect to the $L^2(\Omega)$ metric.\label{thm3}
\end{theorem}

The complete proof of \cref{thm1} and \cref{thm2} can be found in \cref{prof1,prof2} and \cref{thm3} is a direct consequence of \cite[Theorem 4]{bourgain2001another}.

\section{Conclusion and Future Work}
\label{sec:conclusions}

In this paper, we proposed to use both low dimensionality and smoothness of the underlying data manifold as a regularizer for missing data recovery. For that, we introduced curvature regularization (CURE) and a weighted version of it (WeCURE). Comparing to related models such as LDMM, WNLL, and WNTV, the new regularization was proven more effective for semi-supervised learning and image inpainting on some datasets.

There are plenty of future directions worth exploring. For modelling, a natural question is whether different curvatures can also serve as good smoothing regularizers regularizer for data manifolds and how are they different from the one we chose for CURE? Can these curvatures be easily computed? How does CURE work for other tasks of missing data recovery? Furthermore, convergence analysis of solving the Biharmonic equation \eqref{model:wecure:mdr:el} on manifold also needs to be studied. Due to a lack of understanding of the numerical methods for the  Biharmonic equation, it prohibited us from generalizing CURE to generic inverse problems.

\section*{Acknowledgments}
Bin Dong is supported in part by NSFC 11671022 and Beijing Natural Science Foundation (Z180001). Haocheng Ju is supported by the Elite Undergraduate Training Program of the School of Mathematical Sciences at Peking University. Zuoqiang Shi is supported by NSFC
11671005. We would also like to thank Dr. Wei Zhu for his valuable comments and kindly sharing the codes of both LDMM and LDMM+WNLL for comparisons.

\bibliographystyle{siamplain}
\bibliography{mybibtex,1603.00564.pdf,paper_WNLL.pdf}

\begin{thebibliography}{10}

\bibitem{agarwal2006higher}
{\sc S.~Agarwal, K.~Branson, and S.~Belongie}, {\em Higher order learning with
  graphs}, in Proceedings of the 23rd international conference on Machine
  learning, ACM, 2006, pp.~17--24.

\bibitem{bao2016image}
{\sc C.~Bao, B.~Dong, L.~Hou, Z.~Shen, X.~Zhang, and X.~Zhang}, {\em Image
  restoration by minimizing zero norm of wavelet frame coefficients}, Inverse
  problems, 32 (2016), p.~115004.

\bibitem{bertozzi2012diffuse}
{\sc A.~L. Bertozzi and A.~Flenner}, {\em Diffuse interface models on graphs
  for classification of high dimensional data}, Multiscale Modeling \&
  Simulation, 10 (2012), pp.~1090--1118.

\bibitem{bredies2010total}
{\sc K.~Bredies, K.~Kunisch, and T.~Pock}, {\em Total generalized variation},
  SIAM Journal on Imaging Sciences, 3 (2010), pp.~492--526.

\bibitem{buades2006a}
{\sc A.~Buades, B.~Coll, and J.-M. Morel}, {\em Neighborhood filters and
  pde’s}, Numer. Math, 105, p.~1–34.

\bibitem{buades2005a}
{\sc A.~Buades, B.~Coll, and J.-M. Morel}, {\em A review of image denoising
  algorithms, with a new one. multiscale model}, Simul, 4, p.~490–530.

\bibitem{buades2005non}
{\sc A.~Buades, B.~Coll, and J.-M. Morel}, {\em A non-local algorithm for image
  denoising}, in Computer Vision and Pattern Recognition, 2005. CVPR 2005. IEEE
  Computer Society Conference on, vol.~2, IEEE, 2005, pp.~60--65.

\bibitem{burges-a}
{\sc C.~Burges, Y.~LeCun, and C.}, {\em Cortes. mnist database}.

\bibitem{cai2010simultaneous}
{\sc J.-F. Cai, R.~H. Chan, and Z.~Shen}, {\em Simultaneous cartoon and texture
  inpainting}, Inverse Probl. Imaging, 4 (2010), pp.~379--395.

\bibitem{cai2009split}
{\sc J.-F. Cai, S.~Osher, and Z.~Shen}, {\em Split bregman methods and frame
  based image restoration}, Multiscale modeling \& simulation, 8 (2009),
  pp.~337--369.

\bibitem{calder2018properly}
{\sc J.~Calder and D.~Slepcev}, {\em Properly-weighted graph laplacian for
  semi-supervised learning}, arXiv preprint arXiv:1810.04351,  (2018).

\bibitem{chan2003wavelet}
{\sc R.~H. Chan, T.~F. Chan, L.~Shen, and Z.~Shen}, {\em Wavelet algorithms for
  high-resolution image reconstruction}, SIAM Journal on Scientific Computing,
  24 (2003), pp.~1408--1432.

\bibitem{chan2000high}
{\sc T.~Chan, A.~Marquina, and P.~Mulet}, {\em High-order total variation-based
  image restoration}, SIAM Journal on Scientific Computing, 22 (2000),
  pp.~503--516.

\bibitem{chung1997a}
{\sc F.~Chung}, {\em Spectral graph theory}, American Mathematical Society.

\bibitem{coifman2006diffusion}
{\sc R.~R. Coifman and S.~Lafon}, {\em Diffusion maps}, Applied and
  computational harmonic analysis, 21 (2006), pp.~5--30.

\bibitem{dabov2006image}
{\sc K.~Dabov, A.~Foi, V.~Katkovnik, and K.~Egiazarian}, {\em Image denoising
  with block-matching and 3d filtering}, in Image Processing: Algorithms and
  Systems, Neural Networks, and Machine Learning, vol.~6064, International
  Society for Optics and Photonics, 2006, p.~606414.

\bibitem{danielyan2012bm3d}
{\sc A.~Danielyan, V.~Katkovnik, and K.~Egiazarian}, {\em Bm3d frames and
  variational image deblurring}, IEEE Transactions on Image Processing, 21
  (2012), pp.~1715--1728.

\bibitem{daubechies1992ten}
{\sc I.~Daubechies}, {\em Ten lectures on wavelets}, vol.~61, Siam, 1992.

\bibitem{dong2017sparse}
{\sc B.~Dong}, {\em Sparse representation on graphs by tight wavelet frames and
  applications}, Applied and Computational Harmonic Analysis, 42 (2017),
  pp.~452--479.

\bibitem{dong2012mra}
{\sc B.~Dong and Z.~Shen}, {\em Mra-based wavelet frames and applications:
  Image segmentation and surface reconstruction}, in Independent Component
  Analyses, Compressive Sampling, Wavelets, Neural Net, Biosystems, and
  Nanoengineering X, vol.~8401, International Society for Optics and Photonics,
  2012, p.~840102.

\bibitem{Dua:2019}
{\sc D.~Dua and C.~Graff}, {\em {UCI} machine learning repository}, 2017,
  \url{http://archive.ics.uci.edu/ml}.

\bibitem{easley2008sparse}
{\sc G.~Easley, D.~Labate, and W.-Q. Lim}, {\em Sparse directional image
  representations using the discrete shearlet transform}, Applied and
  Computational Harmonic Analysis, 25 (2008), pp.~25--46.

\bibitem{el2016asymptotic}
{\sc A.~El~Alaoui, X.~Cheng, A.~Ramdas, M.~J. Wainwright, and M.~I. Jordan},
  {\em Asymptotic behavior of$\backslash$ell\_p-based laplacian regularization
  in semi-supervised learning}, in Conference on Learning Theory, 2016,
  pp.~879--906.

\bibitem{gilboa2007a}
{\sc G.~Gilboa and S.~Osher}, {\em Nonlocal linear image regularization and
  supervised segmentation}, Multiscale Model. Simul, 6, p.~595–630.

\bibitem{gilboa2008a}
{\sc G.~Gilboa and S.~Osher}, {\em Nonlocal operators with applications to
  image processing}, Multiscale Model. Simul, 7, p.~1005–1028.

\bibitem{gilboa2008nonlocal}
{\sc G.~Gilboa and S.~Osher}, {\em Nonlocal operators with applications to
  image processing}, Multiscale Modeling \& Simulation, 7 (2008),
  pp.~1005--1028.

\bibitem{gu2014weighted}
{\sc S.~Gu, L.~Zhang, W.~Zuo, and X.~Feng}, {\em Weighted nuclear norm
  minimization with application to image denoising}, in Proceedings of the IEEE
  Conference on Computer Vision and Pattern Recognition, 2014, pp.~2862--2869.

\bibitem{lai2018manifold}
{\sc R.~Lai and J.~Li}, {\em Manifold based low-rank regularization for image
  restoration and semi-supervised learning}, Journal of Scientific Computing,
  74 (2018), pp.~1241--1263.

\bibitem{lecun1998mnist}
{\sc Y.~LeCun}, {\em The mnist database of handwritten digits}, http://yann.
  lecun. com/exdb/mnist/,  (1998).

\bibitem{Li2019weighted}
{\sc H.~Li, Z.~Shi, and X.-P. Wang}, {\em Weighted nonlocal total variation in
  image processing}, arXiv preprint, arXiv:1801.10441,  (2019).

\bibitem{li2016a}
{\sc Z.~Li and Z.~Shi}, {\em A convergent point integral method for isotropic
  elliptic equations on point cloud}, SIAM: Multiscale Modeling Simulation, 14,
  p.~874–905.

\bibitem{lim2010discrete}
{\sc W.-Q. Lim}, {\em The discrete shearlet transform: a new directional
  transform and compactly supported shearlet frames.}, IEEE Trans. Image
  Processing, 19 (2010), pp.~1166--1180.

\bibitem{manfio2015minimal}
{\sc F.~Manfio and F.~Vit{\'o}rio}, {\em Minimal immersions of riemannian
  manifolds in products of space forms}, Journal of Mathematical Analysis and
  Applications, 424 (2015), pp.~260--268.

\bibitem{nadler2009semi}
{\sc B.~Nadler, N.~Srebro, and X.~Zhou}, {\em Semi-supervised learning with the
  graph laplacian: The limit of infinite unlabelled data}, Advances in neural
  information processing systems, 22 (2009), pp.~1330--1338.

\bibitem{nene1996columbia}
{\sc S.~A. Nene, S.~K. Nayar, H.~Murase, et~al.}, {\em Columbia object image
  library (coil-20)},  (1996).

\bibitem{osher2016a}
{\sc S.~Osher, Z.~Shi, and W.~Zhu}, {\em Low dimensional manifold model for
  image processing}, technical report, cam report 16-04, UCLA.

\bibitem{peyre2009manifold}
{\sc G.~Peyr{\'e}}, {\em Manifold models for signals and images}, Computer
  Vision and Image Understanding, 113 (2009), pp.~249--260.

\bibitem{roth2005fields}
{\sc S.~Roth and M.~J. Black}, {\em Fields of experts: A framework for learning
  image priors}, in Computer Vision and Pattern Recognition, 2005. CVPR 2005.
  IEEE Computer Society Conference on, vol.~2, IEEE, 2005, pp.~860--867.

\bibitem{rudin1992nonlinear}
{\sc L.~I. Rudin, S.~Osher, and E.~Fatemi}, {\em Nonlinear total variation
  based noise removal algorithms}, Physica D: nonlinear phenomena, 60 (1992),
  pp.~259--268.

\bibitem{shan2008high}
{\sc Q.~Shan, J.~Jia, and A.~Agarwala}, {\em High-quality motion deblurring
  from a single image}, in Acm transactions on graphics (tog), vol.~27, ACM,
  2008, p.~73.

\bibitem{shen2003euler}
{\sc J.~Shen, S.~H. Kang, and T.~F. Chan}, {\em Euler's elastica and
  curvature-based inpainting}, SIAM journal on Applied Mathematics, 63 (2003),
  pp.~564--592.

\bibitem{shi2017weighted}
{\sc Z.~Shi, S.~Osher, and W.~Zhu}, {\em Weighted nonlocal laplacian on
  interpolation from sparse data}, Journal of Scientific Computing, 73 (2017),
  pp.~1164--1177.

\bibitem{starck2002curvelet}
{\sc J.-L. Starck, E.~J. Cand{\`e}s, and D.~L. Donoho}, {\em The curvelet
  transform for image denoising}, IEEE Transactions on image processing, 11
  (2002), pp.~670--684.

\bibitem{stephane1999wavelet}
{\sc M.~Stephane}, {\em A wavelet tour of signal processing}, The Sparse Way,
  (1999).

\bibitem{trillos2016continuum}
{\sc N.~G. Trillos and D.~Slep{\v{c}}ev}, {\em Continuum limit of total
  variation on point clouds}, Archive for rational mechanics and analysis, 220
  (2016), pp.~193--241.

\bibitem{trillos2018variational}
{\sc N.~G. Trillos and D.~Slep{\v{c}}ev}, {\em A variational approach to the
  consistency of spectral clustering}, Applied and Computational Harmonic
  Analysis, 45 (2018), pp.~239--281.

\bibitem{zhang2013l0}
{\sc Y.~Zhang, B.~Dong, and Z.~Lu}, {\em $\ell_0$ minimization for wavelet
  frame based image restoration}, Mathematics of Computation, 82 (2013),
  pp.~995--1015.

\bibitem{zhu1997prior}
{\sc S.~C. Zhu and D.~Mumford}, {\em Prior learning and gibbs
  reaction-diffusion}, IEEE Transactions on Pattern Analysis and Machine
  Intelligence, 19 (1997), pp.~1236--1250.

\bibitem{zhu2018ldmnet}
{\sc W.~Zhu, Q.~Qiu, J.~Huang, R.~Calderbank, G.~Sapiro, and I.~Daubechies},
  {\em Ldmnet: Low dimensional manifold regularized neural networks}, in The
  IEEE Conference on Computer Vision and Pattern Recognition (CVPR), 2018.

\bibitem{zhu2003a}
{\sc X.~Zhu, Z.~Ghahramani, and J.~Lafferty}, {\em Semi-supervised learning
  using gaussian fields and harmonic functions}, in Proceedings of The 31st
  International Conference on Machine Learning, vol.~3, p.~912–919.

\end{thebibliography}

\appendix\section{Preliminaries}
\label{pre}
In this section we present a brief review of some basic concepts used in the asymptotic analysis. The interested readers should consult\cite{trillos2016continuum} for a more detailed introduction to these concepts.
\subsection{Optimal transport}
$\Omega$ is an open and bounded domain in $\mathbb{R}^d$. $\mathscr{B}(\Omega)$ is the Borel $\sigma$-algebra of $\Omega$ and $\mathscr{P}(\Omega)$ is the set of all Borel probability measures on $\Omega$. Given $1\leq p<\infty$, the $p-OT$ distance between $\mu,\hat{\mu}\in \mathscr{P}(\Omega)$ is defined by:
\begin{equation}
d_p(\mu,\hat{\mu}):=\text{min}\left\{\left(\int_{\Omega\times \Omega}|x-y|^pd\pi(x,y)\right)^{1/p}:\pi\in\Gamma(\mu,\hat{\mu})\right\}\label{eq:ot}
\end{equation}
where $\Gamma(\mu,\hat{\mu})$ is the set of all Borel probability
measures on $\Omega\times \Omega$ for which the marginal on the first variable is $\mu$ and the marginal on the second
variable is $\hat{\mu}$. The elements $\pi\in\Gamma(\mu,\hat{\mu})$ are also referred as transportation plans between $\mu$ and $\hat{\mu}$. When $p=\infty$
\begin{equation}
d_{\infty}(\mu,\hat{\mu}):=\text{inf}\left\{\text{esssup}_\pi\left\{|x-y|:(x,y)\in\Omega\times \Omega\right\}: \pi\in\Gamma(\mu,\hat{\mu})\right\}
\end{equation}
defines a metric on $\mathscr{P}(\Omega)$, which is called the $\infty$-transportation distance.\\
Given a Borel map $T:\Omega\rightarrow \Omega$ and $\mu\in\mathscr{P}(\Omega)$ the push-forward of $\mu$ by $T$, denoted by $T_{\sharp}\mu\in \mathscr{P}(\Omega)$ is given by:
\begin{equation}
T_{\sharp}\mu(A):=\mu(T^{-1}(A)),A\in \mathscr{B}(\Omega)
\end{equation}
Then for any bounded Borel function $\varphi:\Omega\rightarrow \mathbb{R}$ the following change of variables in the integral holds:
\begin{equation}
\int_\Omega \varphi(x)d(T_\sharp\mu)(x)=\int_\Omega\varphi(T(x))d\mu(x)
\label{cov1}
\end{equation}
When the measure $\mu\in \mathscr{P}(\Omega)$ is absolutely continuous with respect to the Lebesgue measure, \eqref{eq:ot} is equivalent to:
\begin{equation}
\text{min}\left\{\left(\int_\Omega|x-T(x)|^pd\mu(x)\right)^{1/p}: T_\sharp\mu=\hat{\mu}\right\}
\end{equation}

\subsection{The \texorpdfstring{$TL^p$}{TLp} Space}
The $TL^p$ space was introduced in\cite{trillos2016continuum} to compare functions defined on $\Omega_n=\{x_i:i=1,\cdots,n\}$ and an open domain $\Omega$.
\begin{equation}
TL^p(\Omega)=\{(\mu,f):\mu\in \mathscr{P}(\Omega),f\in L^p(\mu)\}
\end{equation}
The metric on the space is
\begin{equation}
d_{TL^p(\Omega)}^p((\mu,f),(\nu,g))=\text{inf}\left\{\int_{\Omega\times\Omega}|x-y|^p+|f(x)-g(y)|^pd\pi(x,y):\pi\in\Gamma(\mu,\nu)\right\}\label{eq:tl}
\end{equation}
where $\Gamma(\mu,\nu)$ the set of transportation plans defined in the previous subsection.
When the measure $\mu\in \mathscr{P}(\Omega)$ is absolutely continuous with respect to the Lebesgue measure, \eqref{eq:tl} is equivalent to:
\begin{equation}
d_{TL^p(\Omega)}^p((\mu,f),(\nu,g))=\text{inf}\left\{\int_{\Omega\times\Omega}|x-T(x)|^p+|f(x)-g(T(x))|^pd\mu(x):T_\sharp\mu=\nu \right\}
\end{equation}

\subsection{\texorpdfstring{$\Gamma$}{gamma}-Convergence}
We follow the definition of $\Gamma$-convergence by \cite{slepcev2019analysis} in a random setting.
\begin{definition}
Let $(Z,d)$ be a metric space and $(\mathcal{X},\mathbb{P})$ be a probability space. For each $\omega\in\mathcal{X}$ the functional $E_n^{(\omega)}:Z\rightarrow R\cup\{\pm\infty\}$ is a random variable. We say $E_n^{(\omega)}$ $\Gamma$-converge almost surely on the domain $Z$ to $E_\infty:Z\rightarrow R\cup\{\pm\infty\}$ with respect to $d$, and write $E_\infty=\Gamma-\text{lim}_{n\rightarrow\infty}E_n^{(\omega)}$, if there exists a set $\mathcal{X'}\subset\mathcal{X}$ with $\mathbb{P}(\mathcal{X'})=1$, such that for all $\omega\in \mathcal{X'}$ and all $f\in Z$:\\
(\romannumeral1)(liminf inequality) for every sequence $\{f_n\}_{n=1}^{\infty}$ converging to $f$\\
$$
E_\infty(f)\leq \liminf\limits_{n\rightarrow\infty}E_n^{(\omega)}(f_n)
$$
(\romannumeral2)(limsup inequality) there exists a sequence $\{f_n\}_{n=1}^{\infty}$ converging to $f$ such that\\
$$
E_\infty(f)\geq \limsup\limits_{n\rightarrow\infty}E_n^{(\omega)}(f_n)
$$
\end{definition}
\begin{definition}
We say that the sequence of nonnegative functionals $\{F_n\}_{n\in\mathbb{N}}$ satisfies the compactness
property if the following holds: Given $\{n_k\}_{k\in\mathbb{N}}$ an increasing sequence of natural numbers and $\{x_k\}_{k\in\mathbb{N}}$
a bounded sequence in $X$ for which
$$
\sup\limits_{k\in\mathbb{N}}F_{n_k}(x_k)<\infty
$$
$\{x_k\}$ is relatively compact in $X$.
\end{definition}


\subsection{Proof of Theorem 5.2}\label{prof1}
\subsubsection{Liminf inequality}
\begin{proof}
Assume that $u_n\xrightarrow{TL^2}u$ as $n\rightarrow\infty$. First we show that
\begin{equation}
\lim\limits_{n\rightarrow\infty}WeCURE_{n,\varepsilon}(u_n)\eqref{eq:disb}=WeCURE_{\varepsilon}(u)\eqref{n2}
\end{equation}
Since $T_\sharp\nu=\nu_n$, using the change of variables\eqref{cov1} it follows that
\begin{align*}
WeCURE_{n,\varepsilon}(u_n)\eqref{eq:disb}&=\frac{1}{\varepsilon^4}\int_{\Omega}(\int_{\Omega}\eta_\varepsilon(T_n(x)-T_n(y))(u_n\circ T_n(x)-u_n\circ T_n(y))dy)^2dx\\
&=\frac{1}{\varepsilon^4}\int_{\Omega}(\int_{\Omega}\eta_\varepsilon(x-y)(u_n\circ T_n(x)-u_n\circ T_n(y))dy)^2dx+\frac{1}{\varepsilon^4}a_{n}\\
\end{align*}
Notice that 
\begin{align*}
|a_{n}|&=|\int_\Omega(\int_\Omega (\eta_\varepsilon(T_n(x)-T_n(y))-\eta_\varepsilon(x-y))(u_n\circ T_n(x)-u_n\circ T_n(y))dy)\\
&\times(\int_\Omega (\eta_\varepsilon(T_n(x)-T_n(y))+\eta_\varepsilon(x-y))(u_n\circ T_n(x)-u_n\circ T_n(y))dy)dx)|\\
&\triangleq|\int_\Omega F_n(x)G_n(x)dx|\\
&\leq|\int_\Omega F_n^2(x)dx|^{\frac{1}{2}}|\int_\Omega G_n^2(x)dx|^{\frac{1}{2}} 
\end{align*}
Moreover, we have
\begin{align*}
|F_n(x)|&\leq \int_\Omega |\eta_\varepsilon(T_n(x)-T_n(y))-\eta_\varepsilon(x-y)||u_n\circ T_n(x)-u(x)|dy\\
&+\int_\Omega |\eta_\varepsilon(T_n(x)-T_n(y))-\eta_\varepsilon(x-y)||u_n\circ T_n(y)-u(y)|dy\\
&+\int_\Omega |\eta_\varepsilon(T_n(x)-T_n(y))-\eta_\varepsilon(x-y)||u(x)-u(y)|dy
\end{align*}
and
\begin{align*}
\int_\Omega (F_n(x))^2dx&\leq 2\times3\times Area(\Omega)^2\times 4\eta_\varepsilon^2(0)\int_\Omega (u_n\circ T_n(x)-u(x))^2dx\\
&+3\int_\Omega(\int_\Omega (\eta_\varepsilon(T_n(x)-T_n(y))-\eta_\varepsilon(x-y))^2dy)(\int_\Omega (u(x)-u(y))^2dy)dx
\end{align*}
Note that $u_n\xrightarrow{TL^2}u$ indicates $u_n\circ T_n\xrightarrow{L^2(\Omega)}u$, so the first two terms go to zero as $n\rightarrow\infty$. We only have to show
\begin{equation}
\lim\limits_{n\rightarrow\infty}\int_{\Omega}(\eta_\varepsilon(T_n(x)-T_n(y))-\eta_\varepsilon(x-y))^2dy=0
\label{eq:2.1inf}
\end{equation}
Note that for almost every $(x,y)\in \Omega\times \Omega$
\begin{equation}
||x-y|-2\left\|Id-T_n\right\|_\infty|\leq|T_n(y)-T_n(x)|\leq |x-y|+2\left\|Id-T_n\right\|_\infty
\end{equation}
along with the monotonicity of $\eta_\varepsilon$, we have
\begin{align*}
(\eta_\varepsilon(T_n(x)-T_n(y))-\eta_\varepsilon(x-y))^2&\leq \max((\eta_\varepsilon(||x-y|-2\left\|Id-T_n\right\|_\infty|)-\eta_\varepsilon(x-y))^2,\\
&(\eta_\varepsilon(|x-y|+2\left\|Id-T_n\right\|_\infty)-\eta_\varepsilon(x-y))^2)\\
&\leq(\eta_\varepsilon(||x-y|-2\left\|Id-T_n\right\|_\infty|)-\eta_\varepsilon(x-y))^2+\\
&(\eta_\varepsilon(|x-y|+2\left\|Id-T_n\right\|_\infty)-\eta_\varepsilon(x-y))^2
\end{align*}
Note that from Theorem 2.5 in \cite{trillos2016continuum}, we have
\begin{equation}
\lim\limits_{n\rightarrow\infty}\left\|Id-T_n\right\|_\infty=0
\end{equation}
along with the standard result in real analysis that if $f\in L^p(\mathbb{R}^d)$, then $\lim\limits_{h\rightarrow0}\int_{\mathbb{R}^d}|f(r+h)-f(r)|^pdr=0$, we have
\begin{equation}
\lim\limits_{n\rightarrow\infty}\int_{\Omega}(\eta_\varepsilon(|x-y|+2\left\|Id-T_n\right\|_\infty)-\eta_\varepsilon(x-y))^2=0
\end{equation}
Similarly, we can show that
\begin{equation}
\lim\limits_{n\rightarrow\infty}\int_{\Omega}(\eta_\varepsilon(||x-y|-2\left\|Id-T_n\right\|_\infty|)-\eta_\varepsilon(x-y))^2=0
\end{equation}
and we obtain\eqref{eq:2.1inf} and $\lim\limits_{n\rightarrow\infty}a_n=0$, along with
\begin{equation}
\lim\limits_{n\rightarrow\infty}\int_{\Omega}(\int_{\Omega}\eta_\varepsilon(x-y)(u_n\circ T_n(x)-u_n\circ T_n(y))dy)^2dx=\int_{\Omega}(\int_{\Omega}\eta_\varepsilon(x-y)(u(x)-u(y))dy)^2dx
\end{equation}
we have
\begin{equation}
\lim\limits_{n\rightarrow\infty}WeCURE_{n,\varepsilon}(u_n)\eqref{eq:disb}=WeCURE_{\varepsilon}(u)\eqref{n2}
\end{equation}
The rest terms can be proved in a similar way and we have
\begin{equation}
\lim\limits_{n\rightarrow\infty}WeCURE_{n,\varepsilon}(u_n)=WeCURE_{\varepsilon}(u)
\end{equation}
\end{proof}

\subsubsection{Limsup inequality}
\begin{proof}
Define $u_n$ to be the restriction of $u$ to the first $n$ data points $X_1,\cdots,X_n$, and we have $u_n\xrightarrow{TL^2}u$. From the proof of the liminf inequality in the previous section, we have
\begin{equation}
\lim\limits_{n\rightarrow\infty}WeCURE_{n,\varepsilon}(u_n)=WeCURE_{\varepsilon}(u)
\end{equation}
\end{proof}

\subsection{Proof of Theorem 5.3}\label{prof2}
\subsubsection{Liminf inequality}
\begin{proof}
Consider an arbitrary $u\in H_0^2(\Omega)$ and suppose that $u_\varepsilon\xrightarrow{L^2(\Omega)}u$ as $\varepsilon\rightarrow0$
\begin{align*}
&\liminf\limits_{\varepsilon\rightarrow0}WeCURE_{\varepsilon}(u_\varepsilon)\\
&\geq \liminf\limits_{\varepsilon\rightarrow0}WeCURE_\varepsilon(u_\varepsilon)\eqref{n1}\\
&+\liminf\limits_{\varepsilon\rightarrow0}WeCURE_\varepsilon(u_\varepsilon)\eqref{n3}\\
&+\liminf\limits_{\varepsilon\rightarrow0}(WeCURE_\varepsilon(u_\varepsilon)\eqref{n4}+WeCURE_\varepsilon(u_\varepsilon)\eqref{n7})\\
&+\liminf\limits_{\varepsilon\rightarrow0}(WeCURE_\varepsilon(u_\varepsilon)\eqref{n2}+WeCURE_\varepsilon(u_\varepsilon)\eqref{n5}+WeCURE_\varepsilon(u_\varepsilon)\eqref{n6})
\end{align*}
The inequality
\begin{equation}
\liminf\limits_{\varepsilon\rightarrow0}WeCURE_\varepsilon(u_\varepsilon)\eqref{n1} \geq WeCURE(u)\eqref{l2}\label{ine1}
\end{equation}
follows from the proof of Theorem 8 in\cite{ponce2004new}. Next we show that
\begin{equation}
\liminf\limits_{\varepsilon\rightarrow0}WeCURE_\varepsilon(u_\varepsilon)\eqref{n3}\geq WeCURE(u)\eqref{l1}\label{ine2}
\end{equation}
We need the following lemma to establish the liminf inequality.
\begin{lemma}
Let $\Omega$ be a bounded open subset of $\mathbb{R}^d$, $\Omega'$ is a open set compactly contained in $\Omega$. Suppose that $\{u_\varepsilon\}_{\varepsilon>0}$ is a sequence of $C^4$ functions such that
\begin{equation}
\sup\limits_{\varepsilon>0}\left\{\left\|D^4u_\varepsilon\right\|_{L^\infty(\mathbb{R}^d)}\right\}<\infty
\label{asm1}
\end{equation}
if$\Delta u_\varepsilon\xrightarrow{L^2(\Omega)}\Delta u$ for some $u\in C^4(\mathbb{R}^d)$, then
\begin{equation}
\lim\limits_{\varepsilon\rightarrow0}\frac{1}{\varepsilon^4}\int_{\Omega'}(\int_{\Omega}\eta_\varepsilon(x-y)(u_\varepsilon(x)-u_\varepsilon(y))dy)^2dx=\sigma_\eta^2\int_{\Omega'}(\Delta u(x))^2dx
\label{cls1}
\end{equation}
where $\sigma_\eta=\frac{1}{2}\int_{\mathbb{R}^d}\eta(h)|h_1|^2dh$, $h_1$ is the first coordinate of vector $h$.
\label{lem1}
\end{lemma}

\begin{proof}
We claim that
\begin{equation}
\lim\limits_{\varepsilon\rightarrow0}\frac{1}{\varepsilon^4}\int_{\Omega'}(\int_{\Omega}\eta_\varepsilon(x-y)(u_\varepsilon(x)-u_\varepsilon(y))dy)^2dx=\sigma_\eta^2\int_{\Omega'}(\Delta u_\varepsilon(x))^2dx
\label{clm1}
\end{equation}
Using a simple change of variables $h=\frac{y-x}{\varepsilon}$, we have
\begin{align*}
&\frac{1}{\varepsilon^4}\int_{\Omega'}(\int_{\Omega}\eta_\varepsilon(x-y)(u_\varepsilon(x)-u_\varepsilon(y))dy)^2dx\\
&=\frac{1}{\varepsilon^4}\int_{\Omega'}(\int_{x+\varepsilon h\in\Omega}\eta(h)(u_\varepsilon(x)-u_\varepsilon(x+\varepsilon h))dy)^2dx\\
&=\frac{1}{\varepsilon^4}\int_{\Omega'}(\int_{\mathbb{R}^d}\eta(h)(u_\varepsilon(x)-u_\varepsilon(x+\varepsilon h))dh)^2dx\\
&=\frac{1}{\varepsilon^4}\int_{\Omega'}(\int_{\mathbb{R}^d}\eta(h)(\nabla u_\varepsilon(x)\cdot (\varepsilon h)\\
&+\frac{1}{2}(\varepsilon h)^T\cdot\nabla^2 u_\varepsilon(x)\cdot(\varepsilon h))dh)^2dx+C\left\|D^4u_\varepsilon\right\|_{L^\infty(\mathbb{R}^d)}\varepsilon^4\\
&=\sigma_\eta^2\int_{\Omega'}(\Delta u_\varepsilon(x))^2dx+C\left\|D^4u_\varepsilon\right\|_{L^\infty(\mathbb{R}^d)}\varepsilon^4\\
\end{align*}
The second equality follows from that $\Omega'$ is compactly contained in $\Omega$. The third equality follows from fourth order Taylor expansion and the vanishing of first and third order term is a direct result from the radial symmetry of $\eta$. Combined with \eqref{asm1}, we have \eqref{clm1}. Note that $\Delta u_\varepsilon\xrightarrow{L^2(\Omega)}\Delta u$ implies $\left\|\Delta u_\varepsilon\right\|_{L^2(\Omega')}^2\rightarrow\left\|\Delta u\right\|_{L^2(\Omega')}^2$ using H{\"o}lder inequality. Taking $\varepsilon$ to zero in the right hand side of \eqref{clm1} we have \eqref{cls1}.
\end{proof}

We can proceed to the proof of Liminf equality of Theorem 2.2. Our main idea follows from \cite{trillos2016continuum}. Consider an arbitrary $u\in H_0^2(\Omega)$ and suppose that $u_\varepsilon\xrightarrow{L^2(\Omega)}u$ as $\varepsilon\rightarrow0$. We want to show that $\liminf\limits_{\varepsilon\rightarrow0}WeCURE_{\varepsilon}(u_\varepsilon)\geq\sigma_\eta WeCURE(u)$. Without loss of generality, we assume that $\{WeCURE_\varepsilon(u_\varepsilon)\eqref{n3}\}_{\varepsilon>0}$ is uniformly bounded.\\
Consider $J:\mathbb{R}^d\rightarrow [0,\infty)$ a standard mollifier. $J$ is a smooth radially symmetric function, supported in the closed unit ball $\overline{B(0,1)}$ and is such that $\int_{\mathbb{R}^d}J(z)dz=1$. We define $J_\delta(z)=\frac{1}{\delta^d}J(\frac{z}{\delta})$.\\
Fix $\Omega'$ an open domain compactly contained in $\Omega$. Let $\delta'=dist\{\Omega',\partial \Omega\}$. Set $\Omega''=\{x\in \Omega: dist(x,\partial \Omega)>\frac{\delta'}{2}\}$. $\Omega'\subset\subset \Omega''\subset\subset \Omega$. For $0<\delta<\frac{\delta'}{2}$ and for a given function $v\in L^2(\Omega)$ we define the mollified function $v_\delta\in L^1(\mathbb{R}^d)$ by setting $v_\delta(x)=\int_{\mathbb{R}^d}J_\delta(x-z)v(z)dz=\int_{\mathbb{R}^d}J(z)v(x-z)dz$. The functions $v_\delta$ are smooth and satisfy $v_\delta\xrightarrow{L^2(\Omega')}v$ as $\delta\rightarrow0$. Furthermore
\begin{equation}
\nabla v_\delta(x)=\int_{\mathbb{R}^d}\nabla J_\delta(z)v(x-z)dz=\frac{1}{\delta}\int_{\mathbb{R}^d}\frac{1}{\delta^d}\nabla J(\frac{z}{\delta})v(x-z)dz
\label{rep1}
\end{equation}
By taking the second derivative, it follows that there is a constant $C>0$(only depending on the mollifier $J$) such that
\begin{equation}
\left\|D^2v_\delta(x)\right\|_{L^\infty(\mathbb{R}^d)}\leq\frac{C}{\delta^2}\left\|v\right\|_{L^2(\Omega)} \text{and} \left\|D^4v_\delta(x)\right\|_{L^\infty(\mathbb{R}^d)}\leq\frac{C}{\delta^4}\left\|v\right\|_{L^2(\Omega)}
\label{ieq1}
\end{equation}
Since $u_\varepsilon\xrightarrow{L^2(\Omega)}u$ as $\varepsilon\rightarrow0$ the norms $\left\|u_\varepsilon\right\|_{L^2(\Omega)}$ are uniformly bounded. Therefore, taking $v=u_\varepsilon$ in the inequalities\eqref{ieq1} and setting $u_{\varepsilon,\delta}=(u_\varepsilon)_\delta$, implies
$$
\sup\limits_{\varepsilon>0}\left\{\left\|D^4u_{\varepsilon,\delta}\right\|_{L^\infty(\mathbb{R}^d)}\right\}<\infty
$$
Moreover, using \eqref{rep1} to express $D^2u_{\varepsilon,\delta}$ and $D^2u_\delta$, it is straightforward to deduce that
$$
\int_{\Omega'}|D^2(u_{\varepsilon,\delta}-u_\delta)|^2dx\leq\frac{C}{\delta^2}\int_\Omega|u_\varepsilon(x)-u(x)|dx
$$
for some constant $C$ independent of $\varepsilon$. In particular, $\int_{\Omega'}(\Delta(u_{\varepsilon,\delta}-u_\delta))^2dx\rightarrow0$ as $\varepsilon\rightarrow0$ and hence we can apply Lemma\ref{lem1} to infer that\\
\begin{equation}
\lim\limits_{\varepsilon\rightarrow0}\frac{1}{\varepsilon^4}\int_{\Omega'}(\int_{\Omega''}\eta_\varepsilon(x-y)(u_{\varepsilon,\delta}(x)-u_{\varepsilon,\delta}(y))dy)^2dx=\sigma_\eta^2\int_{\Omega'}(\Delta u_\delta(x))^2dx\label{eq:4}
\end{equation}
\begin{align*}
&WeCURE_\varepsilon(u_\varepsilon)\eqref{n3}\\
&\geq\frac{1}{\varepsilon^4}\int_{\Omega''}(\int_{\Omega}\eta_\varepsilon(x-y)(u_\varepsilon(x)-u_\varepsilon(y))dy)^2dx\\
&=\frac{1}{\varepsilon^4}\int_{\mathbb{R}^d}\int_{\Omega''}J_\delta(z)(\int_{\Omega}\eta_\varepsilon(x-y)(u_\varepsilon(x)-u_\varepsilon(y))dy)^2dxdz\\
&\geq \frac{1}{\varepsilon^4}\int_{\mathbb{R}^d}\int_{\Omega'}J_\delta(z)(\int_{\Omega}\eta_\varepsilon(\hat{x}-z-y)(u_\varepsilon(\hat{x}-z)-u_\varepsilon(y))dy)^2d\hat{x}dz\\
&=\frac{1}{\varepsilon^4}\int_{\Omega'}(\int_{\mathbb{R}^d}J_\delta(z)dz)(\int_{\mathbb{R}^d}J_\delta(z)(\int_{\Omega}\eta_\varepsilon(\hat{x}-z-y)(u_\varepsilon(\hat{x}-z)-u_\varepsilon(y))dy)^2dz)d\hat{x}\\
&\geq\frac{1}{\varepsilon^4}\int_{\Omega'}(\int_{\mathbb{R}^d}J_\delta(z)\int_{\Omega}\eta_\varepsilon(\hat{x}-z-y)(u_\varepsilon(\hat{x}-z)-u_\varepsilon(y))dydz)^2d\hat{x}\\
&=\frac{1}{\varepsilon^4}\int_{\Omega'}(\int_{\mathbb{R}^d}\int_{\Omega+\{z\}}\eta_\varepsilon(\hat{x}-\hat{y})J_\delta(z)(u_\varepsilon(\hat{x}-z)-u_\varepsilon(\hat{y}-z))d\hat{y}dz)^2d\hat{x}\\
&=\frac{1}{\varepsilon^4}\int_{\Omega'}(\int_{\mathbb{R}^d}\int_{\Omega''}\eta_\varepsilon(\hat{x}-\hat{y})J_\delta(z)(u_\varepsilon(\hat{x}-z)-u_\varepsilon(\hat{y}-z))d\hat{y}dz)^2d\hat{x}\\
&=\frac{1}{\varepsilon^4}\int_{\Omega'}(\int_{\Omega''}\eta_\varepsilon(\hat{x}-\hat{y})(u_{\varepsilon,\delta}(\hat{x})-u_{\varepsilon,\delta}(\hat{y}))d\hat{y})^2d\hat{x}\\
\end{align*}

The second inequality is obtained by using the change of variables,$\hat{x}=x+z$ and $\Omega'$ is contained in the transformed domain. The third inequality follows from Cauchy-Schwarz inequality. Using a change of variables $\hat{y}=y+z$, we have the third equality. The fourth equality follows from that $\eta$ has compact support, $|z|\leq\delta<\frac{\delta'}{2}$ and thus the integral on $\Omega''$ is the same as the integral on $\Omega+\{z\}$. Let $\varepsilon\rightarrow0$ and apply \eqref{eq:4}, we have
\begin{equation}
\liminf\limits_{\varepsilon\rightarrow0}WeCURE_{\varepsilon}(u_\varepsilon)\eqref{n3}\geq\sigma_\eta^2 \int_{\Omega'}(\Delta u_\delta(x))^2dx
\end{equation}
Since $u_\delta\xrightarrow{L^2(\Omega')}u$ as $\varepsilon\rightarrow0$ and $\int_{\Omega'}(\Delta u(x))^2dx$ is lower semicontinuous, we have
\begin{equation}
\liminf\limits_{\varepsilon\rightarrow0}WeCURE_{\varepsilon}(u_\varepsilon)\eqref{n3}\geq\sigma_\eta^2 \liminf\limits_{\delta\rightarrow0}\int_{\Omega'}(\Delta u_\delta(x))^2dx\geq \sigma_\eta^2\int_{\Omega'}(\Delta u(x))^2dx
\end{equation}
Take $\Omega'\nearrow \Omega$ and we obtain the desired liminf inequality.
Next we show
\begin{equation}
\liminf\limits_{\varepsilon\rightarrow0}(WeCURE_\varepsilon(u_\varepsilon)\eqref{n4}+WeCURE_\varepsilon(u_\varepsilon)\eqref{n7})\geq WeCURE(u)\eqref{l4}
\end{equation}
As $\{WeCURE_\varepsilon(u_\varepsilon)\eqref{n3}\}_{\varepsilon>0}$ is uniformly bounded, we have
\begin{align*}
&\liminf\limits_{\varepsilon\rightarrow0}(WeCURE_\varepsilon(u_\varepsilon)\eqref{n4}+WeCURE_\varepsilon(u_\varepsilon)\eqref{n7})=\\
&\liminf\limits_{\varepsilon\rightarrow0}(\frac{1}{\varepsilon^2}\int_{\Omega}(\int_{\Omega}\eta_\varepsilon(x-y)(u_\varepsilon(x)-u_\varepsilon(y))dy)(u_\varepsilon(x)-b(x))dx\\
&+\frac{1}{\varepsilon^2}\int_{\Omega}(b(x)-u_\varepsilon(x))(\int_{\Omega}\eta_\varepsilon(x-y)(b(x)-b(y))dy)dx)
\end{align*}
Using nonlocal Green's formula in\cite{gilboa2008nonlocal}, we have
\begin{align*}
&\frac{1}{\varepsilon^2}\int_{\Omega}(\int_{\Omega}\eta_\varepsilon(x-y)(u_\varepsilon(x)-u_\varepsilon(y))dy)(u_\varepsilon(x)-b(x))dx\\
&+\frac{1}{\varepsilon^2}\int_{\Omega}(b(x)-u_\varepsilon(x))(\int_{\Omega}\eta_\varepsilon(x-y)(b(x)-b(y))dy)dx\\
&=\frac{1}{\varepsilon^2}\int_{\Omega}\int_{\Omega}\eta_\varepsilon(x-y)(u_\varepsilon(x)-b(x)-u_\varepsilon(y)+b(y))^2dydx
\end{align*}
Substitute $u_\varepsilon-b$ into \eqref{ine1}, we have
\begin{equation}
\liminf\limits_{\varepsilon\rightarrow0}(WeCURE_\varepsilon(u_\varepsilon)\eqref{n4}+WeCURE_\varepsilon(u_\varepsilon)\eqref{n7})\geq WeCURE(u)\eqref{l3}\label{ine3}
\end{equation}
Let $\varepsilon\rightarrow0$, it's straightforward to show
\begin{equation}
\begin{aligned}
&\liminf\limits_{\varepsilon\rightarrow0}(WeCURE_\varepsilon(u_\varepsilon)\eqref{n2}+WeCURE_\varepsilon(u_\varepsilon)\eqref{n5}+WeCURE_\varepsilon(u_\varepsilon)\eqref{n6})\\
&\geq WeCURE(u)\eqref{l4}\label{ine4}
\end{aligned}
\end{equation}
Summing up $\eqref{ine1},\eqref{ine2},\eqref{ine3},\eqref{ine4}$, we have
\begin{equation}
\liminf\limits_{\varepsilon\rightarrow0}WeCURE_\varepsilon(u_\varepsilon)\geq WeCURE(u)
\end{equation}
\end{proof}

\subsubsection{Limsup inequality}
\begin{proof}
From Remark 2.7 in\cite{trillos2016continuum}, we only have to prove the limsup inequality for $u\in C_c^\infty(\Omega)$. We want to prove
\begin{equation}
\limsup\limits_{\varepsilon\rightarrow0} WeCURE_\varepsilon(u)\leq WeCURE(u)
\end{equation}
\begin{align*}
&\limsup\limits_{\varepsilon\rightarrow0}WeCURE_{\varepsilon}(u)\\
&\leq \limsup\limits_{\varepsilon\rightarrow0}WeCURE_\varepsilon(u)\eqref{n1}\\
&+\limsup\limits_{\varepsilon\rightarrow0}WeCURE_\varepsilon(u)\eqref{n3}\\
&+\limsup\limits_{\varepsilon\rightarrow0}(WeCURE_\varepsilon(u)\eqref{n4}+WeCURE_\varepsilon(u)\eqref{n7})\\
&+\limsup\limits_{\varepsilon\rightarrow0}(WeCURE_\varepsilon(u)\eqref{n2}+WeCURE_\varepsilon(u)\eqref{n5}+WeCURE_\varepsilon(u)\eqref{n6})
\end{align*}
The inequality
\begin{equation}
\limsup\limits_{\varepsilon\rightarrow0}WeCURE_\varepsilon(u)\eqref{n1} \leq WeCURE(u)\eqref{l2}\label{ineq1}
\end{equation}
follows from the proof of Theorem 8 in\cite{ponce2004new}. Next we show
\begin{equation}
\limsup\limits_{\varepsilon\rightarrow0}WeCURE_\varepsilon(u)\eqref{n3}\leq WeCURE(u)\eqref{l1}\label{ineq2}
\end{equation}
Let $\Omega_\varepsilon=\{x\in \Omega:dist(x,\partial \Omega)>\alpha\varepsilon\}$.
\begin{align*}
&\frac{1}{\varepsilon^4}\int_{\Omega_\varepsilon}(\int_{\Omega}\eta_\varepsilon(x-y)(u(x)-u(y))dy)^2dx\\
&=\frac{1}{\varepsilon^4}\int_{\Omega_\varepsilon}(\int_{B(x,\alpha\varepsilon)}\eta_\varepsilon(x-y)(y-x)^T\cdot\int_0^1\int_0^p\nabla^2u(x+t(y-x))dtdp\cdot(y-x)dy)^2dx\\
&\leq \int_{\Omega}(\int_{|h|<\alpha}\eta(h)h^T\cdot\int_0^1\int_0^p\nabla^2u(z)dtdp\cdot hdh)^2dz\\
&=\frac{1}{4}\int_{\Omega}(\int_{|h|<\alpha}\eta(h)h^T\cdot\nabla^2u(z)\cdot hdh)^2dz\\
&=\sigma_\eta^2\int_{\Omega}(\Delta u(z))^2dz
\end{align*}

The first equality is obtained by setting $F(t)=v_k(x+t(y-x))-v_k(x),F(1)-F(0)=\int_0^1\int_0^tF''(p)dpdt+f'(0)$ and the vanishing of first order term is a direct result from the radial symmetry of $\eta$. $\nabla^2$ stands for the Hessian matrix. The first inequality is obtained by a change of variables $(y,x)\rightarrow(h,z),h=\frac{y-x}{\varepsilon},z=x+t(y-x)$ and the transformed domain is contained in $\Omega$. As $u$ is compactly supported, it's straightforward to show that
\begin{equation}
\lim\limits_{\varepsilon\rightarrow0}\frac{1}{\varepsilon^4}\int_{\Omega\backslash \Omega_\varepsilon}(\int_{\Omega}\eta_\varepsilon(x-y)(u(x)-u(y))dy)^2dx=0
\end{equation}
then we have
\begin{equation}
\limsup\limits_{\varepsilon\rightarrow0}\frac{1}{\varepsilon^4}\int_{\Omega}(\int_{\Omega}\eta_\varepsilon(x-y)(u(x)-u(y))dy)^2dx\leq\sigma_\eta^2\int_{\Omega}(\Delta u(z))^2dz
\end{equation}
Similar to the proof of inequality\eqref{ine3}, we have
\begin{equation}
\limsup\limits_{\varepsilon\rightarrow0}(WeCURE_\varepsilon(u)\eqref{n4}+WeCURE_\varepsilon(u)\eqref{n7})\leq WeCURE(u)\eqref{l3}\label{ineq3}
\end{equation}
Let $\varepsilon\rightarrow0$, it's straightforward to show
\begin{equation}
\begin{aligned}
&\limsup\limits_{\varepsilon\rightarrow0}(WeCURE_\varepsilon(u)\eqref{n2}+WeCURE_\varepsilon(u)\eqref{n5}+WeCURE_\varepsilon(u)\eqref{n6})\\
&\leq WeCURE(u)\eqref{l4}\label{ineq4}
\end{aligned}
\end{equation}
Summing up $\eqref{ineq1},\eqref{ineq2},\eqref{ineq3},\eqref{ineq4}$, we have
\begin{equation}
\limsup\limits_{\varepsilon\rightarrow0}WeCURE_\varepsilon(u)\leq WeCURE(u)
\end{equation}

\end{proof}

\end{document}